\documentclass{article}

 \usepackage[final]{nips_2017}

\usepackage[utf8]{inputenc} \usepackage[T1]{fontenc}    \usepackage{url}            \usepackage{booktabs}       \usepackage{amsfonts}       \usepackage{nicefrac}       \usepackage{microtype}      \usepackage{times}
\usepackage{mathtools}
\usepackage{graphicx} \usepackage{subcaption}

\usepackage{algorithm}
\usepackage{algorithmic}
\usepackage{multirow}
\usepackage{amsthm}
\usepackage{mathrsfs}
\usepackage{graphicx}
\usepackage{caption}\usepackage{xspace}
\usepackage{float}
\usepackage{wrapfig}
\usepackage{enumitem}
\usepackage{wrapfig}
\setitemize{noitemsep,topsep=0pt,parsep=2pt,partopsep=0pt}
\usepackage{natbib}

\usepackage{multibib} \newcites{sup}{Supplementary References}
\usepackage{hyperref}       \usepackage[font = footnotesize]{caption}

\def\R{{\mathbb{R}}}

\def\to{{\,\rightarrow\,}}

\mathchardef\mhyphen="2D

\newcommand{\norm}[1]{{ \left\lVert#1\right\rVert }}

\newcommand{\vertiii}[1]{{\left\vert\kern-0.25ex\left\vert\kern-0.25ex\left\vert #1
    \right\vert\kern-0.25ex\right\vert\kern-0.25ex\right\vert}}

\newcommand{\vect}[1]{{\boldsymbol{#1}}}

\def\vA{{\vect{A}}}

\def\ba{{\mathbf{a}}}
\def\bb{{\mathbf{b}}}

\def\bd{{\mathbf{d}}}
\def\be{{\mathbf{e}}}

\def\br{{\mathbf{r}}}
\def\bs{{\mathbf{s}}}

\def\bv{{\mathbf{v}}}

\def\bx{{\mathbf{x}}}
\def\by{{\mathbf{y}}}
\def\bz{{\mathbf{z}}}
\def\bM{{\mathbf{M}}}

\def\0{{\mathbf{0}}}

\def\bbR{{\mathbb{R}}}

\def\cA{\mathcal{A}}
\def\cB{\mathcal{B}}

\def\cH{\mathcal{H}}

\def\cK{\mathcal{K}}

\def\cQ{\mathcal{Q}}

\def\cS{\mathcal{S}}
\def\cT{\mathcal{T}}

\newtheorem*{rep@theorem}{\rep@title}
\newcommand{\newreptheorem}[2]{%
\newenvironment{rep#1}[1]{%
 \def\rep@title{#2 \ref{##1}}%
 \begin{rep@theorem}}%
 {\end{rep@theorem}}}
\newreptheorem{lemma}{Lemma'}
\newreptheorem{definition}{Definition'}
\newreptheorem{proposition}{Proposition'}
\newreptheorem{theorem}{Theorem'}

\newtheorem{theorem}{Theorem}
\newtheorem{lemma}[theorem]{Lemma}
\renewcommand{\text}[1]{{\textnormal{#1}}}

\DeclareMathOperator{\cone}{cone}
\DeclareMathOperator*{\argmin}{arg\,min}
\DeclareMathOperator*{\argmax}{arg\,max}
\DeclareMathOperator{\conv}{conv}
\DeclareMathOperator{\faces}{faces}
\DeclareMathOperator{\gfaces}{g-faces}
\DeclareMathOperator{\lin}{lin}
\DeclareMathOperator{\diam}{\mathrm{diam}}
\DeclareMathOperator{\radius}{\mathrm{radius}}

\newcommand{\lmo}{\textsc{LMO}\xspace}

\DeclareMathOperator{\mdw}{mDW(\cA)}

\DeclareMathOperator{\cw}{CWidth(\cA)}

\newcommand{\MP}{{\textsf{\tiny MP}}}

\newcommand{\domain}{\cQ}
\newcommand{\Cf}{C_{f,\cA}}

\newcommand{\CfMP}{C_{f,\cA}^\MP}
\newcommand{\CfMPr}{C_{f,\rho (\cA\cup-\cA)}^\MP}
\newcommand{\CfMPrPW}{C_{f,\rho (\cA\cup-\cA)}^\text{A}}
\newcommand{\CfMPPW}{C_{f,\cA}^\text{A}}

\newcommand{\mufr}{\mu_{f, \rho \cA}^A}

\title{Greedy Algorithms for Cone Constrained Optimization with Convergence Guarantees}

\author{
  Francesco Locatello\\
MPI for Intelligent Systems - ETH Zurich\\
  \texttt{\small{locatelf@ethz.ch}} \\
  \And
  Michael Tschannen\\
  ETH Zurich\\
  \texttt{\small{michaelt@nari.ee.ethz.ch}}\\
  \AND
  Gunnar R{\"a}tsch\\
  ETH Zurich\\
  \texttt{\small{raetsch@inf.ethz.ch}}
  \And 
  Martin Jaggi\\
  EPFL\\
  \texttt{\small{martin.jaggi@epfl.ch}}
 }

\begin{document}

\maketitle

\begin{abstract} 
Greedy optimization methods such as Matching Pursuit (MP) and Frank-Wolfe (FW) algorithms regained popularity in recent years due to their simplicity, effectiveness and theoretical guarantees. 
MP and FW address optimization over the \textit{linear span} and the \textit{convex hull} of a set of atoms, respectively. 
In this paper, we consider the intermediate case of optimization over the \textit{convex cone}, parametrized as the conic hull of a generic atom set, leading to the first principled definitions of non-negative MP algorithms for which we give explicit convergence rates and demonstrate excellent empirical performance. In particular, we derive sublinear ($\mathcal{O}(1/t)$) convergence on general smooth and convex objectives, and linear convergence ($\mathcal{O}(e^{-t})$) on strongly convex objectives, in both cases for general sets of atoms. Furthermore, we establish a clear correspondence of our algorithms to known algorithms from the MP and FW literature. Our novel algorithms and analyses target general atom sets and general objective functions, and hence are directly applicable to a large variety of learning settings.  
\end{abstract} 

\section{Introduction}

In recent years, greedy optimization algorithms have attracted significant interest in the domains of signal processing and machine learning thanks to their ability to process very large data sets. Arguably two of the most popular representatives are Frank-Wolfe (FW) \cite{Frank1956bt,Jaggi:2013wg} and Matching Pursuit (MP) algorithms \cite{Mallat:1993gu}, in particular Orthogonal MP (OMP) \cite{chen1989orthogonal, Tropp:2004gc}. While the former targets minimization of a convex function over \textit{bounded convex sets}, the latter apply to minimization over a \textit{linear subspace}. In both cases, the domain is commonly parametrized by a set of atoms or dictionary elements, and in each iteration, both algorithms rely on querying a so-called \emph{linear minimization oracle} (\lmo) to find the direction of steepest descent in the set of atoms. The iterate is then updated as a \textit{linear} or \textit{convex combination}, respectively, of previous iterates and the newly obtained atom from the \lmo. The particular choice of the atom set allows to encode structure such as sparsity and non-negativity (of the atoms) into the solution. This enables control of the trade-off between the amount of structure in the solution and approximation quality via the number of iterations, which was found useful in a large variety of use cases including structured matrix and tensor factorizations \cite{wang2014matrixcompletion,Yang:2015wy, yaogreedy,guo2017efficient}.

In this paper, we target an important ``intermediate case'' between the two domain parameterizations given by the \textit{linear span} and the \textit{convex hull} of an atom set, namely the parameterization of the optimization domain as the \emph{conic hull} of a possibly infinite atom set. In this case, the solution can be represented as a \emph{non-negative} linear combination of the atoms, which is desirable in many applications, e.g., due to the physics underlying the problem at hand, or for the sake of interpretability. Concrete examples include unmixing problems \cite{esser2013method, gillis2016fast,behr2013mitie}, model selection \cite{makalic2011logistic}, and matrix and tensor factorizations \cite{berry2007algorithms, kim2012fast}. However, existing convergence analyses do not apply to the currently used greedy algorithms. 
In particular, all existing MP variants for the conic hull case \cite{bruckstein2008uniqueness,ID52513,Yaghoobi:2015ff} are not guaranteed to converge and may get stuck far away from the optimum (this can be observed in the experiments in Section \ref{sec:experiments}). 
From a theoretical perspective, this intermediate case is of paramount interest in the context of MP and FW algorithms. Indeed, the atom set is not guaranteed to contain an atom aligned with a descent direction for all possible suboptimal iterates, as is the case when the optimization domain is the linear span or the convex hull of the atom set \cite{pena2015polytope,locatello2017unified}. 
Hence, while conic constraints have been widely studied in the context of a manifold of different applications, none of the existing greedy algorithms enjoys explicit convergence rates. 

We propose and analyze new MP algorithms tailored for the minimization of smooth convex functions over the conic hull of an atom set. Specifically, our key contributions are:
\begin{itemize}
\item We propose the first (non-orthogonal) MP algorithm for optimization over conic hulls guaranteed to converge, and prove a corresponding \textit{sublinear} convergence rate with \textit{explicit constants}. Surprisingly, convergence is achieved without increasing computational complexity compared to ordinary MP.
\item We propose new away-step, pairwise, and fully corrective MP variants, inspired by variants of FW \cite{LacosteJulien:2015wj} and generalized MP \cite{locatello2017unified}, respectively, that allow for different degrees of weight corrections for previously selected atoms. We derive corresponding sublinear and linear (for strongly convex objectives) convergence rates that solely depend on the geometry of the atom set. 
\item All our algorithms apply to general smooth convex functions. This is in contrast to all prior work on non-negative MP, which targets quadratic objectives \cite{bruckstein2008uniqueness, ID52513, Yaghoobi:2015ff}. Furthermore, if the conic hull of the atom set equals its linear span, we recover both algorithms and rates derived in \cite{locatello2017unified} for generalized MP variants. \item We make no assumptions on the atom set which is simply a subset of a Hilbert space, in particular we do not assume the atom set to be finite. \end{itemize}

Before presenting our algorithms (Section~\ref{sec:generalgreedy}) along with the corresponding convergence guarantees (Section~\ref{sec:convrates}), we briefly review generalized MP variants. A detailed discussion of related work can be found in Section~\ref{sec:relwork} followed by illustrative experiments on a least squares problem on synthetic data, and non-negative matrix factorization as well as non-negative garrote logistic regression as applications examples on real data (numerical evaluations of more applications and the dependency between constants in the rate and empirical convergence can be found in the supplementary material).

\paragraph{Notation.} 
Given a non-empty subset $\cA$ of some Hilbert space, let $\conv(\cA)$ be the convex hull of~$\cA$, and let $\lin(\cA)$ denote its linear span. Given a closed set $\cA$, we call its diameter $\diam(\cA)=\max_{\bz_1,\bz_2\in\cA}\|\bz_1-\bz_2\|$ and its radius $\radius(\cA) = \max_{\bz\in\cA}\|\bz\|$.  $\|\bx\|_\cA := \inf \{ c > 0 \colon \bx \in c \cdot \conv (\cA) \}$ is the atomic norm of $\bx$ over a set $\cA$ (also known as the gauge function of $\conv (\cA)$). We call a subset $\cA$ of a Hilbert space symmetric if it is closed under negation.
\section{Review of Matching Pursuit Variants} \label{sec:mprev}
Let $\cH$ be a Hilbert space with associated inner product $\langle \bx, \by\rangle, \,\forall \, \bx,\by \in \cH$. The inner product induces the norm $\| \bx \|^2 := \langle \bx, \bx \rangle,$ $\forall \, \bx \in \cH$. Let $\cA \subset \cH$ be a compact set (the ``set of atoms'' or dictionary) and let $f \colon \cH \to \bbR$ be convex and $L$-smooth ($L$-Lipschitz gradient in the finite dimensional case). If $\cH$ is an infinite-dimensional Hilbert space, then $f$ is assumed to be \emph{Fr{\'e}chet differentiable}. The generalized MP algorithm studied in  \cite{locatello2017unified}, presented in Algorithm~\ref{algo:generalgreedy}, solves the following optimization problem:\vspace{-2mm}
\begin{equation}\label{eq:MPproblem}
\min_{\bx\in\lin(\cA)} f(\bx).
\end{equation}
\vspace{-2mm}
\begin{wrapfigure}{L}{0.51\textwidth}
    \begin{minipage}{0.51\textwidth}
    \vspace{-0.5cm}
\begin{algorithm}[H]
\caption{Norm-Corrective Generalized Matching Pursuit}
  \label{algo:generalgreedy}
\begin{algorithmic}[1]
  \STATE \textbf{init} $\bx_{0} \in \lin(\cA)$, and $\cS:=\{\bx_{0}\}$
  \STATE \textbf{for} {$t=0\dots T$}
  \STATE \quad Find $\bz_t := (\text{Approx-}) \lmo_\cA(\nabla f(\bx_{t}))$
  \STATE \quad $\cS:=\cS\cup\{ \bz_t \}$\
  \STATE \quad $\text{ }$Let $ \bb := \bx_{t}-\frac1L \nabla f(\bx_{t})$ \\
\STATE \quad Variant 0:  \\ \qquad Update $\bx_{t+1}:= \displaystyle\argmin_{\substack{\bz := \bx_{t} + \gamma \bz_t \\ \gamma \in \R}} \!\|{\bz-\bb}\|^2$\\
\STATE \quad Variant 1: \\ \qquad Update $\bx_{t+1}:= \displaystyle\argmin_{\bz\in \lin(\cS)} \|{\bz-\bb}\|^2$
  \STATE \quad \emph{Optional:} Correction of some/all atoms $\bz_{0\ldots t}$
  \STATE \textbf{end for}
\end{algorithmic}
\end{algorithm}
\vspace{-0.2mm}
\end{minipage}
\end{wrapfigure}

In each iteration, MP queries a linear minimization oracle (\lmo) solving the following linear problem:
\begin{equation}\label{eq:FWLMO}
\lmo_\cA(\by) := \argmin_{\bz\in\cA} \,\langle \by, \bz \rangle
\end{equation}
for a given query $\by\in\cH$. The MP update step minimizes a quadratic upper bound $
 g_{\bx_{t}}(\bx) = f(\bx_{t}) + \langle\nabla f(\bx_{t}), \bx-\bx_{t}\rangle+\frac{L}{2}\|\bx-\bx_{t}\|^2
$
of $f$ at $\bx_t$, where $L$ is an upper bound on the smoothness constant of $f$ with respect to a chosen norm $\|\cdot\|$. Optimizing this norm problem instead of $f$ directly allows for substantial efficiency gains in the case of complicated $f$. 
For symmetric $\cA$ and for $f(\bx) = \frac{1}{2}\|\by -\bx\|^2$, $\by \in \cH$, Algorithm~\ref{algo:generalgreedy} recovers MP (Variant~0) \cite{Mallat:1993gu} and OMP (Variant~1) \cite{chen1989orthogonal, Tropp:2004gc}, see \cite{locatello2017unified} for details. 
\paragraph{Approximate linear oracles.} 
\label{sec:approxlmo}
Solving the \lmo defined in \eqref{eq:FWLMO} exactly is often hard in practice, in particular when applied to matrix (or tensor) factorization problems, while approximate versions can be much more efficient. Algorithm~\ref{algo:generalgreedy} allows for an \emph{approximate \lmo}. For given quality parameter $\delta\in \left( 0,1\right]$ and given direction $\bd\in\cH$, the approximate \lmo for Algorithm~\ref{algo:generalgreedy} returns a vector $\tilde{\bz}\in\cA$ such that
\begin{equation}\label{eq:inexactLMOMP}
 \langle \bd,\tilde{\bz}\rangle \leq \delta\langle \bd,\bz\rangle, \end{equation} 
relative to $\bz =\lmo_\cA(\bd)$ being an exact solution.
\paragraph{Discussion and limitations of MP.}
The analysis of the convergence of Algorithm~\ref{algo:generalgreedy} in \cite{locatello2017unified} critically relies on the assumption that the origin is in the relative interior of $\conv(\cA)$ with respect to its linear span. 
This assumption originates from the fact that the convergence of MP- and FW-type algorithms fundamentally depends on an \emph{alignment assumption} of the search direction returned by the \lmo (i.e., $\bz_t$ in Algorithm~\ref{algo:generalgreedy}) and the gradient of the objective at the current iteration (see \textit{third premise} in \cite{pena2015polytope}). Specifically, for Algorithm~\ref{algo:generalgreedy}, the $\lmo$ is assumed to select a descent direction, i.e., $\langle \nabla f(\bx_t), \bz_t \rangle < 0$, so that the resulting weight (i.e., $\gamma$ for Variant 0) is always positive. In this spirit, Algorithm~\ref{algo:generalgreedy} is a natural candidate to minimize $f$ over the conic hull of $\cA$. 
However, if the optimization domain is a cone, the alignment assumption does not hold as there may be non-stationary points $\bx$ in the conic hull of $\cA$ for which $\min_{\bz\in\cA}\langle \nabla f(\bx),\bz\rangle = 0$. Algorithm~\ref{algo:generalgreedy} is therefore not guaranteed to converge when applied to conic problems. The same issue arises for essentially all existing non-negative variants of MP, see, e.g., Alg. 2 in \cite{ID52513} and in Alg. 2 in \cite{Yaghoobi:2015ff}. We now present modifications corroborating this issue along with the resulting MP-type algorithms for conic problems and corresponding convergence guarantees.
\section{Greedy Algorithms on Conic Hulls}
\label{sec:generalgreedy}
The cone $\cone(\cA-\by)$ tangent to the convex set $\conv(\cA)$ at a point $\by$  is formed by the half-lines emanating from $\by$ and intersecting $\conv(\cA)$ in at least one point distinct from $\by$. 
Without loss of generality we consider $\0\in\cA$ and assume the set $\cone(\cA)$ (i.e., $\by=\0$) to be closed. If $\cA$ is finite the cone constraint can be written as
$
\cone(\cA):=\lbrace\bx: \bx = \sum_{i=1}^{|\cA|}\alpha_i\ba_i \;\;  \text{s.t.} \; \ba_i\in\cA, \ \alpha_i\geq 0 ~\forall i\rbrace\!.
$
We consider conic optimization problems of the form:
\begin{align} \label{eq:coneprob}
\min_{\bx\in\cone(\cA)} f(\bx). 
\end{align} 
Note that if the set $\cA$ is symmetric or if the origin is in the relative interior of $\conv(\cA)$ w.r.t. its linear span then $\cone(\cA) = \lin(\cA)$. We will show later how our results recover known MP rates when the origin is in the relative interior of $\conv(\cA)$.

As a first algorithm to solve problems of the form~\eqref{eq:coneprob}, we present the Non-Negative Generalized Matching Pursuit (NNMP) in Algorithm~\ref{algo:NNMP} which is an extension of MP to general $f$ and non-negative weights.

\paragraph{Discussion:}
Algorithm~\ref{algo:NNMP} differs from Algorithm~\ref{algo:generalgreedy} (Variant 0) in line~4, adding the iteration-dependent atom $-\frac{\bx_t}{\|\bx_t\|_\cA}$ to the set of possible search directions\footnote{This additional direction makes sense only if $\bx_t\neq \0$. Therefore, we set $-\frac{\bx_t}{\|\bx_t\|_\cA}=0$ if $\bx_t=0$, i.e., no direction is added.}. We use the atomic norm for the normalization because it yields the best constant in the convergence rate. In practice, one can replace it with the Euclidean norm, which is often much less expensive to compute.  This iteration-dependent additional search direction allows to reduce the weights of the atoms that were previously selected, thus admitting the algorithm to ``move back'' towards the origin while maintaining the cone constraint. This idea is informally explained here and formally studied in Section~\ref{sec:sublin}. 

\begin{figure}
\begin{minipage}{0.47\textwidth}
    \vspace{-0.3cm}
\begin{algorithm}[H]
\caption{Non-Negative Matching Pursuit}
\label{algo:NNMP}
\begin{algorithmic}[1]
  \STATE \textbf{init} $\bx_{0} = \0 \in \cA$
  \STATE \textbf{for} {$t=0\dots T$}
  \STATE \quad Find $\bar{\bz}_t := (\text{Approx-}) \lmo_{\cA}(\nabla f(\bx_{t}))$
   \STATE \quad $\bz_t = \argmin_{\bz\in \left\lbrace\bar{\bz}_t,\frac{-\bx_t}{\|\bx_t\|_\cA}\right\rbrace}\langle \nabla f(\bx_t),\bz\rangle$ 
  \STATE \quad $\gamma := \frac{\langle -\nabla f(\bx_t), \bz_t   \rangle}{L\|\bz_t\|^2}$
  \STATE \quad   Update $\bx_{t+1}:= \bx_t + \gamma \bz_t$
  \STATE \textbf{end for}
\end{algorithmic}
\end{algorithm}
\vspace{-0.3cm}
\end{minipage}
\hspace{2mm}
\begin{minipage}{0.5\textwidth}
\vspace{0mm}
\includegraphics[scale=0.32]{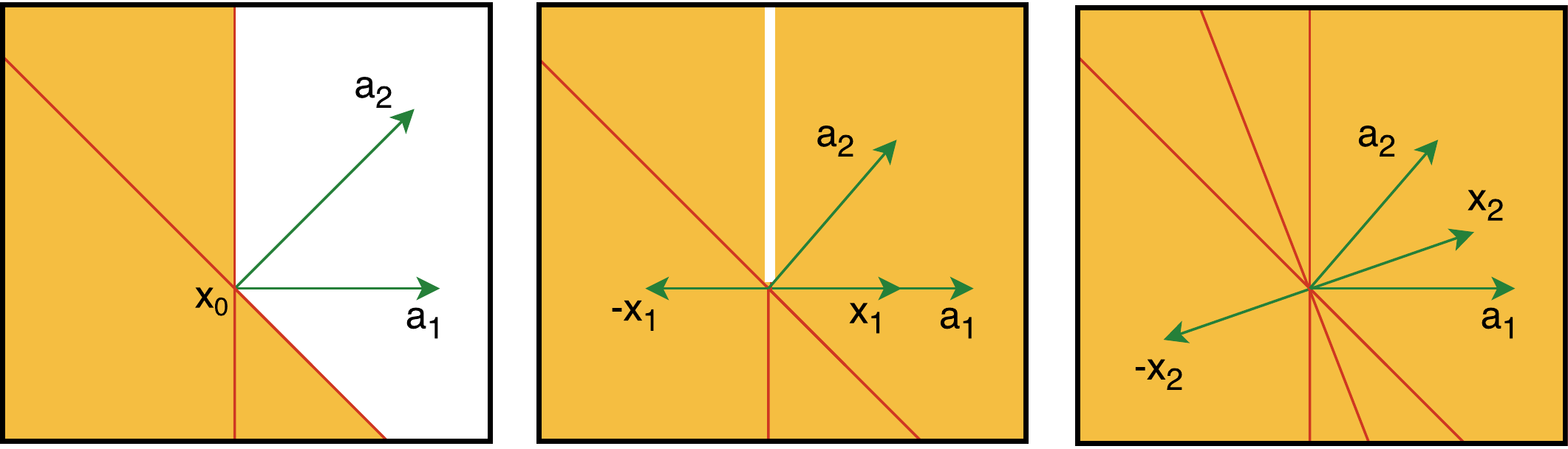}
\hspace{2mm}\captionof{figure}{Two dimensional example for $T_\cA(\bx_t)$ where $\cA = \lbrace \ba_1,\ba_2\rbrace$, for three different iterates $\bx_0,\ \bx_1$ and $\bx_2$. The shaded area corresponds to $T_\cA(\bx_t)$ and the white area to $\lin(\cA)\setminus T_\cA(\bx_t)$.}\label{pic:cone}
\vspace{-2mm}
\end{minipage}
\vspace{-5mm}
\end{figure}

Recall the alignment assumption of the search direction and the gradient of the objective at the current iterate discussed in Section~\ref{sec:mprev} (see also \cite{pena2015polytope}). 
Algorithm \ref{algo:NNMP} obeys this assumption. The intuition behind this is the following. Whenever $\bx_t$ is not a minimizer of~\eqref{eq:coneprob} and $\min_{\bz\in\cA}\langle \nabla f(\bx_t),\bz\rangle = 0$, the vector $-\frac{\bx_t}{\|\bx_t\|_\cA}$ is aligned with $\nabla f(\bx_t)$ (i.e., $\langle \nabla f(\bx_t),-\frac{\bx_t}{\|\bx_t\|_\cA} \rangle < 0$), preventing the algorithm from stopping at a suboptimal iterate.
To make this intuition more formal, let us define the set of feasible descent directions of Algorithm~\ref{algo:NNMP} at a point $\bx\in \cone(\cA)$ as:
\begin{equation}
T_\cA(\bx)  :=  \left\lbrace  \bd\in\cH\!: \exists \bz\in\cA\cup \Big\lbrace  -\frac{\bx}{\|\bx\|_\cA} \Big\rbrace \l \ \text{s.t.} \; \langle \bd,\bz\rangle<0  \right\rbrace .
\end{equation}
If at some iteration $t=0,1,\ldots$ the gradient $\nabla f(\bx_t)$ is not in $T_\cA(\bx_t)$ Algorithm~\ref{algo:NNMP} terminates as $\min_{\bz\in\cA}\langle \bd,\bz\rangle = 0$ and $\langle \bd,-\bx_t\rangle\geq 0$ (which yields $\bz_t = 0$).
Even though, in general, not every direction in $\cH$ is a feasible descent direction, $\nabla f(\bx_t) \notin T_\cA$ only occurs if $\bx_t$ is a constrained minimum of Equation~\ref{eq:coneprob}:

\begin{lemma}\label{lemma:originOpt}
If $\tilde\bx \in \cone(\cA)$ and $\nabla f(\tilde\bx)\not\in T_\cA$ then $\tilde\bx$ is a solution to $\min_{\bx\in\cone(\cA)}f(\bx)$.
\end{lemma}
\vspace{-2mm}

Initializing Algorithm~\ref{algo:NNMP} with $\bx_0 = \0$ guarantees that the iterates $\bx_t$ always remain inside $\cone(\cA)$ even though this is not enforced explicitly (by convexity of $f$, see proof of Theorem~\ref{thm:NNMPsublinear} in Appendix~\ref{sec:sublinpf} for details).
\paragraph{Limitations of Algorithm \ref{algo:NNMP}:}
Let us call \textit{active} the atoms which have nonzero weights in the representation of $\bx_t = \sum_{i=0}^{t-1} \alpha_i \bz_i$ computed by Algorithm~\ref{algo:NNMP}. Formally, the set of active atoms is defined as $\cS := \{ \bz_i \colon \alpha_i > 0, i = 0,1,\ldots,t-1 \}$.
The main drawback of Algorithm~\ref{algo:NNMP} is that when the direction $-\frac{\bx_t}{\|\bx_t\|_\cA}$ is selected, the weight of \textit{all} active atoms is reduced. 
This can lead to the algorithm alternately selecting $-\frac{\bx_t}{\|\bx_t\|_\cA}$ and an atom from $\cA$, thereby slowing down convergence in a similar manner as the \textit{zig-zagging} phenomenon well-known in the Frank-Wolfe framework \cite{LacosteJulien:2015wj}. In order to achieve faster convergence we introduce the corrective variants of Algorithm~\ref{algo:NNMP}.
\subsection{Corrective Variants}
To achieve faster (linear) convergence (see Section~\ref{sec:lin}) we introduce variants of Algorithm~\ref{algo:NNMP}, termed Away-steps MP (AMP) and Pairwise MP (PWMP), presented in Algorithm~\ref{algo:AMPPWMP}. 
Here, inspired by the away-steps and pairwise variants of FW \cite{Frank1956bt,LacosteJulien:2015wj}, instead of reducing the weights of the active atoms uniformly as in Algorithm~\ref{algo:NNMP}, the $\lmo$ is queried a second time on the active set $\cS$ to identify the direction of steepest ascent in $\cS$. This allows, at each iteration, to reduce the weight of a previously selected atom (AMP) or swap weight between atoms (PWMP). This selective ``reduction'' or ``swap of weight'' helps to avoid the zig-zagging phenomenon which prevent Algorithm~\ref{algo:NNMP} from converging linearly.

At each iteration, Algorithm~\ref{algo:AMPPWMP} updates the weights of $\bz_t$ and $\bv_t$ as $\alpha_{\bz_t}=\alpha_{\bz_t}+\gamma$ and $\alpha_{\bv_t}=\alpha_{\bv_t}-\gamma$, respectively. To ensure that $\bx_{t+1}\in\cone(\cA)$, $\gamma$ has to be clipped according to the weight which is currently on $\bv_t$, i.e., $\gamma_{\max}=\alpha_{\bv_t}$. If $\gamma = \gamma_{\max}$, we set $\alpha_{\bv_t}=0$ and remove $\bv_t$ from $\cS$ as the atom~$\bv_t$ is no longer active. If $\bd_t \in\cA$ (i.e., we take a regular MP step and not an away step), the line search is unconstrained (i.e., $\gamma_{\max}=\infty$).

For both algorithm variants, the second $\lmo$ query increases the computational complexity. Note that an exact search on $\cS$ is feasible in practice as $|\cS|$ has at most $t$ elements at iteration $t$. 

Taking an additional computational burden allows to update the weights of all active atoms in the spirit of OMP. This approach is implemented in the Fully Corrective MP (FCMP), Algorithm \ref{algo:FCNNMP}.

    \begin{minipage}[t]{0.49\textwidth}
    \vspace{-0.5cm}
\begin{algorithm}[H]
\caption{Away-steps (AMP) and Pairwise (PWMP) Non-Negative Matching Pursuit}
\label{algo:AMPPWMP}
\begin{algorithmic}[1]
  \STATE \textbf{init} $\bx_{0} = \0 \in \cA$, and $\cS:=\{\bx_{0}\}$
  \STATE \textbf{for} {$t=0\dots T$}
  \STATE \quad Find $\bz_t := (\text{Approx-}) \lmo_{\cA}(\nabla f(\bx_{t}))$
  \STATE \quad Find $\bv_t := (\text{Approx-}) \lmo_{\cS}(-\nabla f(\bx_{t}))$
  \STATE \quad $\cS = \cS \cup \bz_t$
  \STATE \quad \textit{AMP:} $\bd_t \!=\! \argmin_{\bd\in\lbrace\bz_t,-\bv_t\rbrace} \!\langle\nabla f(\bx_t),\bd \rangle\!$
  \STATE \quad \textit{PWMP: } $\bd_t = \bz_t-\bv_t$
  \STATE \quad $\gamma := \min\left\lbrace\frac{\langle -\nabla f(\bx_t),\bd_t\rangle}{L\|\bd_t\|^2},\gamma_{\max}\right\rbrace$ \\ \qquad($\gamma_{\max}$ see text)
  \STATE \quad Update $\alpha_{\bz_t}$, $\alpha_{\bv_t}$ and $\cS$ according to $\gamma$ \\ \qquad ($\gamma$ see text)
  \STATE \quad   Update $\bx_{t+1}:= \bx_t + \gamma\bd_t$
  \STATE \textbf{end for}
\end{algorithmic}
\end{algorithm}
\end{minipage}
\hspace{0.2cm}
    \begin{minipage}[t]{0.49\textwidth}
    \vspace{-0.5cm}
\begin{algorithm}[H]
\caption{Fully Corrective Non-Negative Matching Pursuit (FCMP)}
\label{algo:FCNNMP}
\begin{algorithmic}[1]
  \STATE \textbf{init} $\bx_{0} = \0 \in \cA, \cS = \{\bx_0 \}$
  \STATE \textbf{for} {$t=0\dots T$}
  \STATE \quad Find $\bz_t := (\text{Approx-}) \lmo_{\cA}(\nabla f(\bx_{t}))$
  \STATE \quad $\cS := \cS \cup \{ \bz_t\}$
\STATE \quad \textit{Variant 0:} \\
 \qquad $\bx_{t+1} \! =\! \underset{{\bx\in\cone(\cS)}}{\argmin} \!\|\bx - (\bx_t -\frac{1}{L}\nabla f(\bx_t))\|^2\!$
\STATE \quad \textit{Variant 1:} \\ \qquad $\bx_{t+1} = \argmin_{\bx\in\cone(\cS)} f(\bx)$
\STATE \quad Remove atoms with zero weights from $\cS$
  \STATE \textbf{end for}
\end{algorithmic}
\end{algorithm}
\end{minipage}
\vspace{1mm}

At each iteration, Algorithm \ref{algo:FCNNMP} maintains the set of active atoms $\cS$ by adding $\bz_t$ and removing atoms with zero weights after the update. 
In Variant 0, the algorithm minimizes the quadratic upper bound $g_{\bx_{t}}(\bx)$ on $f$ at $\bx_{t}$ (see Section \ref{sec:mprev}) imitating a gradient descent step with projection onto a ``varying'' target, i.e., $\cone(\cS)$.  
In Variant~1, the original objective $f$ is minimized over $\cone(\cS)$ at each iteration, which is in general more efficient than minimizing $f$ over $\cone(\cA)$ using a generic solver for cone constrained problems. 
For $f(\bx) = \frac{1}{2}\|\by -\bx\|^2$, $\by \in \cH$, Variant 1 recovers Algorithm~1 in \cite{Yaghoobi:2015ff} and the OMP variant in \cite{bruckstein2008uniqueness} which both only apply to this specific objective $f$.

\subsection{Computational Complexity} \label{sec:compcompl}
\begin{wraptable}{r}{8.8cm}\vspace{-0.6cm}
\center\begin{tabular}{ l | c c c }
  \emph{algorithm} & \emph{cost per iteration} & \emph{convergence} & $k(t)$ \\
  \hline
  NNMP & $C+O(d)$ & $O(1/t)$& - \\
  PWMP & $C+O(d+td)$ & $O\left(e^{-\beta k(t)}\right)$&$\frac{t}{3|\cA|!+1}$\\
  AMP & $C+O(d+td)$ & $O\left(e^{-\frac\beta2 k(t)}\right)$& $t/2$\\
  FCMP v. 0 & $C+O(d)+h_0$ & $O\left(e^{-\beta k(t)}\right)$& $\frac{t}{3|\cA|!+1}$\\
  FCMP v. 1 & $C+O(d)+h_1$ & $O\left(e^{-\beta k(t)}\right)$& $t$\\
\end{tabular}
\caption{Computational complexity versus convergence rate (see Section \ref{sec:convrates}) for strongly convex objectives}\label{table:complexity}
\vspace{-2mm}
\end{wraptable}
We briefly discuss the computational complexity of the algorithms we introduced. For $\cH = \bbR^d$, sums and inner products have cost $O(d)$. 
Let us assume that each call of the $\lmo$ has cost $C$ on the set $\cA$ and $O(td)$ on $\cS$.
The variants 0 and 1 of FCMP solve a cone problem at each iteration with cost $h_0$ and $h_1$, respectively. In general, $h_0$ can be much smaller than $h_1$.
In Table~\ref{table:complexity} we report the cost per iteration for every algorithm along with the asymptotic convergence rates derived in Section~\ref{sec:convrates}. 
\vspace{-2mm}

\section{Convergence Rates} \label{sec:convrates}
In this section, we present convergence guarantees for Algorithms~\ref{algo:NNMP}, \ref{algo:AMPPWMP}, and \ref{algo:FCNNMP}. All proofs are deferred to the Appendix in the supplementary material.
We  write $\bx^\star \in \argmin_{\bx\in \cone(\cA)} f(\bx)$ for an optimal solution.
Our rates will depend on the atomic norm of the solution and the iterates of the respective algorithm variant: 
\begin{equation}\label{eq:rho}
\rho = \max\left\lbrace \|\bx^\star\|_{\cA}, \|\bx_{0}\|_{\cA}\ldots,\|\bx_T\|_{\cA}\right\rbrace.
\end{equation}
If the optimum is not unique, we consider $\bx^\star$ to be one of largest atomic norm. 
A more intuitive and looser notion is to simply upper-bound $\rho$ by the diameter of the level set of the initial iterate $\bx_0$ measured by the atomic norm. Then, boundedness follows since the presented method is a descent method (due to Lemma \ref{lemma:originOpt} and line search on the quadratic upper bound, each iteration strictly decreases the objective and our method stops only at the optimum). This justifies the statement $f(\bx_t ) \leq f(\bx_0)$. Hence, $\rho$ must be bounded for any sequence of iterates produced by the algorithm, and the convergence rates presented in this section are valid as $T$ goes to infinity. A similar notion to measure the convergence of MP was established in \cite{locatello2017unified}.
All of our algorithms and rates can be made \textit{affine invariant}. We defer this discussion to Appendix~\ref{sec:AffInv}.
\subsection{Sublinear Convergence}\label{sec:sublin}
We now present the convergence results for the non-negative and Fully-Corrective Matching Pursuit algorithms. Sublinear convergence of Algorithm~\ref{algo:AMPPWMP} is addressed in Theorem~\ref{thm:PWMPlinear}.

\begin{theorem}\label{thm:NNMPsublinear}
Let $\cA \subset \cH$ be a bounded set with $\0\in\cA$, $\rho := \max\left\lbrace \|\bx^\star\|_{\cA}, \|\bx_{0}\|_{\cA},\ldots,\|\bx_T\|_{\cA},\right\rbrace $ and $f$ be $L$-smooth over $\rho\conv(\cA\cup-\cA)$.
Then, Algorithms~\ref{algo:NNMP} and \ref{algo:FCNNMP} converge for $t \geq 0$ as 
\[
f(\bx_t) - f(\bx^\star)\leq \frac{4\left(\frac2\delta L\rho^2\radius(\cA)^2+\varepsilon_0\right)}{\delta t+4} ,
\]
where $\delta \in (0,1]$ is the relative accuracy parameter of the employed approximate \lmo (see Equation~\eqref{eq:inexactLMOMP}).
\end{theorem}
\paragraph{Relation to FW rates.}
By rescaling $\cA$ by a large enough factor $\tau>0$, FW with $\tau \cA$ as atom set could in principle be used to solve \eqref{eq:coneprob}. In fact, for large enough $\tau$, only the constraints of \eqref{eq:coneprob} become active when minimizing $f$ over $\conv(\tau \cA)$. The sublinear convergence rate obtained with this approach is up to constants identical to that in Theorem~\ref{thm:NNMPsublinear} for our MP variants, see \cite{Jaggi:2013wg}. However, as the correct scaling is unknown, one has to either take the risk of choosing $\tau$ too small and hence failing to recover an optimal solution of \eqref{eq:coneprob}, or to rely on too large $\tau$ which can result in slow convergence. In contrast, knowledge of $\rho$ is not required to run our MP variants.
\paragraph{Relation to MP rates.}
If $\cA$ is symmetric, we have that $\lin(\cA)=\cone(\cA)$ and it is easy to show that the additional direction $-\frac{\bx_t}{\|\bx_t\|}$ in Algorithm~\ref{algo:NNMP} is never selected. Therefore, Algorithm~\ref{algo:NNMP} becomes equivalent to Variant 0 of Algorithm~\ref{algo:generalgreedy}, while Variant 1 of Algorithm~\ref{algo:generalgreedy} is equivalent to Variant 0 of Algorithm~\ref{algo:FCNNMP}. The rate specified in Theorem~\ref{thm:NNMPsublinear} hence generalizes the sublinear rate in \cite[Theorem 2]{locatello2017unified} for symmetric $\cA$. 

\subsection{Linear Convergence} \label{sec:lin}
We start by recalling some of the geometric complexity quantities that were introduced in the context of FW and are adapted here to the optimization problem we aim to solve (minimization over $\cone(\cA)$ instead of $\conv(\cA)$). 
\vspace{-2mm}
\paragraph{Directional Width.}
The directional width of a set $\cA$ w.r.t. a direction $\br\in\cH$ is defined as:
\begin{align}
dirW(\cA,\br):=\max_{\bs,\bv\in\cA}\big\langle\tfrac{\br}{\|\br\|},\bs-\bv\big\rangle
\end{align}
\vspace{-5mm}
\paragraph{Pyramidal Directional Width  {\normalfont\cite{LacosteJulien:2015wj}}.}
The Pyramidal Directional Width of a set $\cA$ with respect to a direction $\br$ and a reference point $\bx\in\conv(\cA)$ is defined as:
\begin{align}
PdirW(\cA,\br,\bx) := \min_{\cS\in\cS_\bx}dirW(\cS\cup
\lbrace\bs(\cA,\br)\rbrace,\br),
\end{align}
where $\cS_\bx := \lbrace \cS \ |\ \cS\subset \cA$ and $\bx$ is a proper convex combination of all the elements in $\cS\rbrace$ and $\bs(\cA,\br) := \max_{\bs\in\cA}\langle\frac{\br}{\|\br\|},\bs\rangle$.

Inspired by the notion of pyramidal width in \cite{LacosteJulien:2015wj}, which is the minimal pyramidal directional width computed over the set of feasible directions, we now define the cone width of a set $\cA$ where only the generating faces ($\gfaces$) of $\cone(\cA)$ (instead of the faces of $\conv(\cA)$) are considered. Before doing so we introduce the notions of \textit{face}, \textit{generating face}, and \textit{feasible direction}.
\paragraph{Face of a convex set.} Let us consider a set $\cK$ with a $k-$dimensional affine hull along with a point $\bx\in\cK$. Then, $\cK$ is a $k-$dimensional face of $\conv(\cA)$ if $\cK = \conv(\cA) \cap \lbrace \by \colon \langle \br,\by-\bx\rangle = 0\rbrace$ for some normal vector $\br$ and $\conv(A)$ is contained in the half-space determined by $\br$, i.e., $\langle\br,\by-\bx\rangle\leq 0$, $\forall \ \by\in\conv(\cA)$.
Intuitively, given a set $\conv(\cA)$ one can think of  $\conv(\cA)$ being a $\mathrm{dim}(\conv(\cA))-$dimensional face of itself, an edge on the border of the set a $1$-dimensional face and a vertex a $0$-dimensional face.
\vspace{-2mm}
\paragraph{Face of a cone and g-faces.}
Similarly, a $k-$dimensional face of a cone is an open and unbounded set $\cone(\cA) \cap \lbrace \by \colon \langle \br,\by-\bx\rangle = 0\rbrace$ for some normal vector $\br$ and $\cone(A)$ is contained in the half space determined by $\br$.  We can define the generating faces of a cone as:
\begin{align*}
\gfaces(\cone(\cA)) \! := \! \left\lbrace \cB \cap \conv(\cA)\colon \! \cB\in \faces(\cone(\cA))\right\rbrace.
\end{align*}
Note that $\gfaces(\cone(\cA))\subset \faces(\conv(\cA))$ and $\conv(\cA)\in\gfaces(\cone(\cA))$. Furthermore, for each $\cK\in\gfaces(\cone(\cA))$, $\cone(\cK)$ is a $k-$dimensional face of $\cone(\cA)$.

We now introduce the notion of \textbf{feasible directions}. A direction $\bd$ is feasible from $\bx \in \cone(\cA)$ if it points inwards $\cone(\cA)$, i.e., if $\exists \varepsilon > 0$ s.t. $\bx + \varepsilon \bd \in \cone(\cA)$.
Since a face of the cone is itself a cone, if a direction is feasible from $\bx\in\cone(\cK)\setminus \0$, it is feasible from every positive rescaling of $\bx$. We therefore can consider only the feasible directions on the generating faces (which are closed and bounded sets). Finally, we define the cone width of $\cA$.
\paragraph{Cone Width.}\vspace{-1mm}
\begin{align} \label{def:conewidth}
\cw:= \min_{\substack{\cK\in \gfaces(\cone(\cA))\\ \bx\in \cK \\ \br\in\cone(\cK-\bx)\setminus \lbrace \0\rbrace}} PdirW(\cK\cap\cA,\br,\bx)
\end{align}

We are now ready to show the linear convergence of Algorithms~\ref{algo:AMPPWMP} and \ref{algo:FCNNMP}.
\begin{theorem}\label{thm:PWMPlinear}
Let $\cA \subset \cH$ be a bounded set with $\0\in\cA$ and let the objective function $f \colon \cH \to \R$ be both $L$-smooth and $\mu$-strongly convex over $\rho \conv(\cA\cup-\cA)$. 
Then, the suboptimality of the iterates of Algorithms~\ref{algo:AMPPWMP} and \ref{algo:FCNNMP} decreases geometrically at each step in which $\gamma < \alpha_{\bv_t}$ (henceforth referred to as ``good steps'') as:
\begin{equation} \label{eq:linrate}
\varepsilon_{t+1}
\leq \left(1- \beta \right)\varepsilon_{t},
\end{equation}
where $\beta := \delta^2\frac{\mu \cw^2}{L\diam(\cA)^2}\in (0,1]$, $\varepsilon_t := f(\bx_t) - f(\bx^\star)$ is the suboptimality at step $t$ and $\delta \in (0,1]$ is the relative accuracy parameter of the employed approximate \lmo \eqref{eq:inexactLMOMP}. For AMP (Algorithm~\ref{algo:AMPPWMP}), $\beta^{\text{AMP}} = \beta/2$. If $\mu = 0$ Algorithm~\ref{algo:AMPPWMP} converges with rate $O(1/k(t))$ where $k(t)$ is the number of ``good steps'' up to iteration t.
\end{theorem}
\vspace{-2mm}
\paragraph{Discussion.} To obtain a linear convergence rate, one needs to upper-bound the number of ``bad steps'' $t-k(t)$ (i.e., steps with $\gamma \geq \alpha_{\bv_t}$). We have that $k(t)=t$ for Variant 1 of FCMP (Algorithm~\ref{algo:FCNNMP}), $k(t)\geq t/2$ for AMP (Algorithm~\ref{algo:AMPPWMP}) and $k(t) \geq t/(3|\cA|!+1)$ for PWMP (Algorithm~\ref{algo:AMPPWMP}) and Variant 0 of FCMP (Algorithm~\ref{algo:FCNNMP}). This yields a global linear convergence rate of $\varepsilon_t\leq \varepsilon_0 \exp\left(-\beta k(t)\right)$. The bound for PWMP is very loose and only meaningful for finite sets $\cA$. However, it can be observed in the experiments in the supplementary material (Appendix~\ref{app:experments}) that only a very small fraction of iterations result in bad PWMP steps in practice. Further note that Variant 1 of FCMP (Algorithm \ref{algo:FCNNMP}) does not produce bad steps. Also note that the bounds on the number of good steps given above are the same as for the corresponding FW variants and are obtained using the same (purely combinatorial) arguments as in \cite{LacosteJulien:2015wj}. \vspace{-2mm}
\paragraph{Relation to previous MP rates.}
The linear convergence of the generalized (not non-negative) MP variants studied in  \cite{locatello2017unified} crucially depends on the geometry of the set which is characterized by the Minimal Directional Width $\mdw$:
\begin{align}
\mdw := \min_{ \substack {\bd\in\lin(\cA)\\\bd \neq \0}}\max_{\bz\in\cA}\langle \frac{\bd}{\|\bd\|},\bz\rangle \ . \end{align}
The following Lemma relates the Cone Width with the minimal directional width.
\begin{lemma} \label{lem:mdw}
If the origin is in the relative interior of $\conv(\cA)$ with respect to its linear span, then $\cone(\cA)=\lin(\cA)$ and $\cw= \mdw$.
\end{lemma}
Now, if the set $\cA$ is symmetric or, more generally, if $\cone(\cA)$ spans the linear space $\lin(\cA)$ (which implies that the origin is in the relative interior of $\conv(\cA)$), there are no bad steps. Hence, by Lemma \ref{lem:mdw}, the linear rate obtained in Theorem~\ref{thm:PWMPlinear} for non-negative MP variants generalizes the one presented in \cite[Theorem~7]{locatello2017unified} for generalized MP variants.

\vspace{-2mm}
\paragraph{Relation to FW rates.}
Optimization over conic hulls with non-negative MP is more similar to FW than to MP itself in the following sense. For MP, every direction in $\lin(\cA)$ allows for unconstrained steps, from any iterate $\bx_t$. In contrast, for our non-negative MPs, while some directions allow for unconstrained steps from some iterate $\bx_t$, others are constrained, thereby leading to the dependence of the linear convergence rate on the cone width, a geometric constant which is very similar in spirit to 
the Pyramidal Width appearing in the linear convergence bound in \cite{LacosteJulien:2015wj} for FW. Furthermore, as for Algorithm~\ref{algo:AMPPWMP}, the linear rate of Away-steps and Pairwise FW holds only for good steps. We finally relate the cone width with the Pyramidal Width \cite{LacosteJulien:2015wj}. The Pyramidal Width is defined as 
\begin{align*}
\mathrm{PWidth}(\cA):= \min_{\substack{\cK\in \faces(\conv(\cA))\\ \bx\in \cK \\ \br\in\cone(\cK-\bx)\setminus \lbrace \0\rbrace}} PdirW(\cK\cap\cA,\br,\bx).
\end{align*}
We have $\cw \geq \mathrm{PWidth}(\cA)$ as the minimization in the definition \eqref{def:conewidth} of $\cw$ is only over the subset $\gfaces(\cone(\cA))$ of $\faces(\conv(\cA))$. As a consequence, the decrease per iteration characterized in Theorem~\ref{thm:PWMPlinear} is larger than what one could obtain with FW on the rescaled convex set $\tau\cA$ (see Section \ref{sec:sublin} for details about the rescaling). Furthermore, the decrease characterized in \cite{LacosteJulien:2015wj} scales as $1/\tau^2$ due to the dependence on $1/\diam(\conv(\cA))^2$. 
\section{Related Work} \label{sec:relwork}

The line of recent works by \cite{ShalevShwartz:2010wq, Temlyakov:2013wf, Temlyakov:2014eb, Temlyakov:2012vg, nguyen2014greedy, locatello2017unified} targets the generalization of MP from the least-squares objective to general smooth objectives and derives corresponding convergence rates (see \cite{locatello2017unified} for a more in-depth discussion). 
However, only little prior work targets MP variants with non-negativity constraint \cite{bruckstein2008uniqueness,ID52513,Yaghoobi:2015ff}. In particular, the least-squares objective was addressed and no rigorous convergence analysis was carried out. \cite{bruckstein2008uniqueness,Yaghoobi:2015ff} proposed an algorithm equivalent to our Algorithm~\ref{algo:FCNNMP} for the least-squares case. More specifically, \cite{Yaghoobi:2015ff} then developed an acceleration heuristic, whereas \cite{bruckstein2008uniqueness} derived a coherence-based recovery guarantee for sparse linear combinations of atoms. Apart from MP-type algorithms, there is a large variety of non-negative least-squares algorithms, e.g.,~\cite{lawson1995solving}, in particular also for matrix and tensor spaces. 
The gold standard in factorization problems is projected gradient descent with alternating minimization, see \cite{Sha:2002um,berry2007algorithms,shashua2005non,kim2014algorithms}. Other related works are~\cite{pena2016solving}, which is concerned with the feasibility problem on symmetric cones, and \cite{harchaoui2015conditional}, which introduces a norm-regularized variant of problem \eqref{eq:coneprob} and solves it using FW on a rescaled convex set. To the best of our knowledge, in the context of MP-type algorithms, we are the first to combine general convex objectives with conic constraints and to derive corresponding convergence guarantees. 
\paragraph{Boosting:}
In an earlier line of work, a flavor of the generalized MP became popular in the context of boosting, see \cite{meir2003introduction}. The literature on boosting is vast, we refer to \cite{ratsch2001convergence,meir2003introduction,buhlmann2010boosting} for a general overview.
Taking the optimization perspective given in \cite{ratsch2001convergence}, boosting is an iterative greedy algorithm minimizing a (strongly) convex objective over the linear span of a possibly infinite set called hypothesis class. The convergence analysis crucially relies on the assumption of the origin being in the relative interior of the hypothesis class, 
see Theorem 1 in \cite{grubb2011generalized}. Indeed, Algorithm 5.2 of \cite{meir2003introduction} might not converge 
if the~\cite{pena2015polytope} alignment assumption is violated.
Here, we managed to relax this assumption while preserving essentially the same asymptotic rates in \cite{meir2003introduction,grubb2011generalized}. Our work is therefore also relevant in the context of (non-negative) boosting.
\section{Illustrative Experiments} \label{sec:experiments}
\vspace{-2mm}
We illustrate the performance of the presented algorithms on three different exemplary tasks, showing that our algorithms are competitive with established baselines across a wide range of objective functions, domains, and data sets while not being specifically tailored to any of these tasks 
(see Section~\ref{sec:compcompl} for a discussion of the computational complexity of the algorithms).
Additional experiments targeting KL divergence NMF, non-negative tensor factorization, and hyperspectral image unmixing can be found in the appendix.

\paragraph{Synthetic data.} 
We consider minimizing the least squares objective on the conic hull of~100 unit-norm vectors sampled at random in the first orthant of $\bbR^{50}$. We compare the convergence of Algorithms~\ref{algo:NNMP}, \ref{algo:AMPPWMP}, and \ref{algo:FCNNMP} with the Fast Non-Negative MP (FNNOMP) of~\cite{Yaghoobi:2015ff}, and Variant 3 (line-search) of the FW algorithm in \cite{locatello2017unified} on the atom set rescaled by $\tau = 10 \|\by\|$ (see Section~\ref{sec:sublin}), observing linear convergence for our corrective variants. 
\begin{wrapfigure}{l}{0.4\textwidth}
\vspace{-3mm}
\includegraphics[scale=0.3]{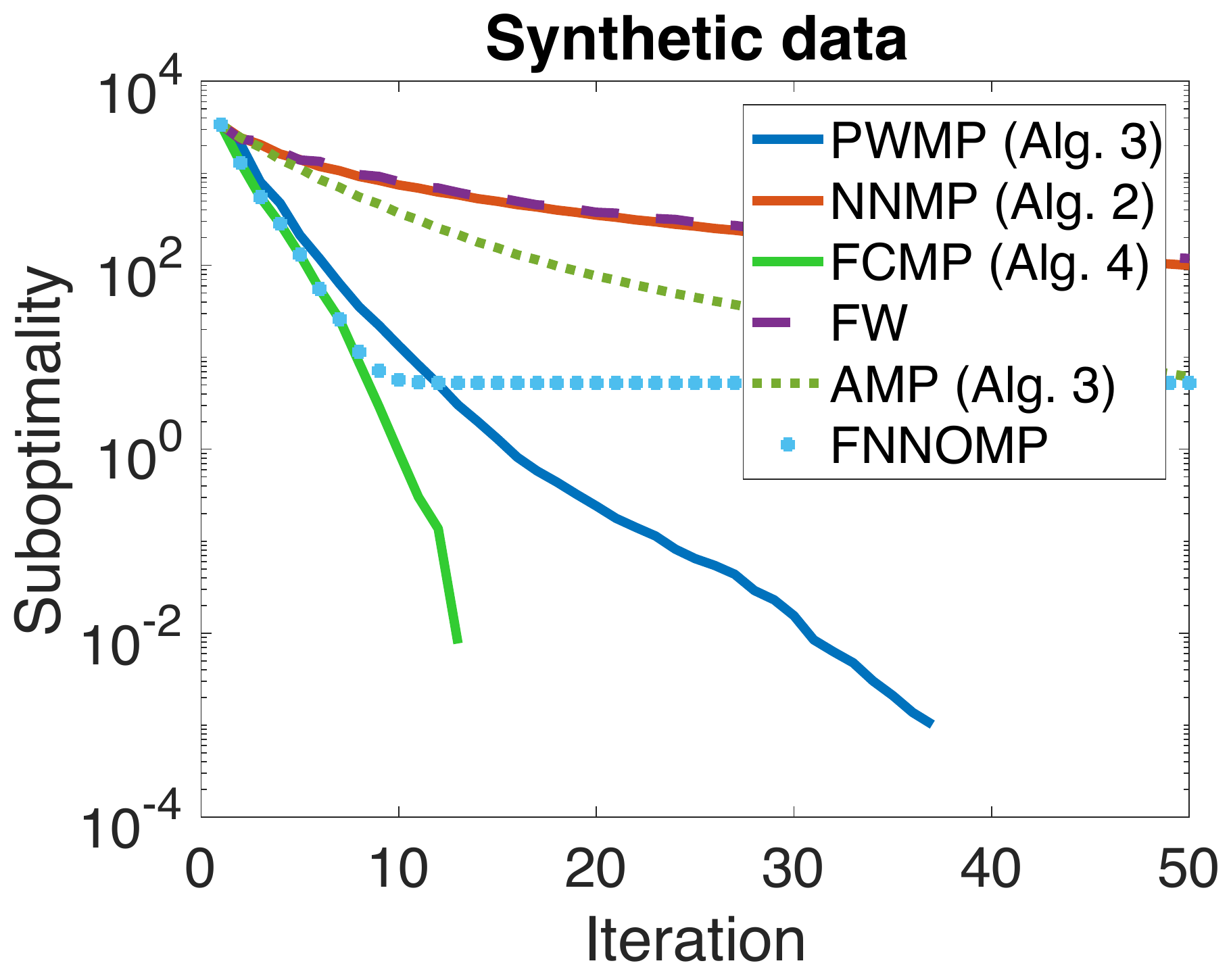}
\vspace{-5mm}
\caption{Synthetic data experiment.}\label{exp:synth}
\vspace{-5mm}
\end{wrapfigure}
Figure~\ref{exp:synth} shows the suboptimality $\varepsilon_t$, averaged over 20 realizations of $\cA$ and $\by$, as a function of the iteration $t$. As expected, FCMP achieves fastest convergence followed by PWMP, AMP and NNMP. 
The FNNOMP gets stuck instead. Indeed, \citesup{Yaghoobi:2015ff} only show that the algorithm terminates and not its convergence.
\paragraph{Non-negative matrix factorization.} The second task consists of decomposing a given matrix into the product of two non-negative matrices as in Equation~(1) of \cite{hsieh2011fast}.
We consider the intersection of the positive semidefinite cone and the positive orthant.
We parametrize the set $\cA$ as the set of matrices obtained as an outer product of vectors from $\cA_1 = \lbrace \bz\in\bbR^{k}: \bz_i \geq 0 \ \forall \ i \rbrace$ and $\cA_2 = \lbrace \bz\in\bbR^{d}: \bz_i \geq 0 \ \forall \ i \rbrace$. The \lmo is approximated using a truncated power method \cite{yuan2013}, and we perform atom correction with greedy coordinate descent see, e.g., \cite{Laue:2012wn, guo2017efficient}, to obtain a better objective value while maintaining the same (small) number of atoms.
We consider three different datasets: The Reuters Corpus\footnote{\url{http://www.nltk.org/book/ch02.html}}, the CBCL face dataset\footnote{\url{http://cbcl.mit.edu/software-datasets/FaceData2.html}} and the KNIX dataset\footnote{\url{http://www.osirix-viewer.com/resources/dicom-image-library/}}.
The subsample of the Reuters corpus we used is a term frequency matrix of 7,769 documents and 26,001 words.
The CBCL face dataset is composed of 2,492 images of 361 pixels each, arranged into a matrix. 
The KNIX dataset contains 24 MRI slices of a knee, arranged in a matrix of size $262,144\times 24$. Pixels are divided by their overall mean intensity. For interpretability reasons, there is interest to decompose MRI data into non-negative factorizations~\citesup{kopriva2010nonlinear}.  We compare PWMP and FCMP against the multiplicative (mult) and the alternating (als) algorithm of \cite{berry2007algorithms}, and the greedy coordinate descent (GCD) of \cite{hsieh2011fast}. 
Since the Reuters corpus is much larger than the CBCL and the KNIX dataset we only used the GCD for which a fast implementation in C is available. 
We report the objective value for fixed values of the rank in Table~\ref{table:nnmfText}, showing that FCMP outperform all the baselines across all the datasets. PWMP achieves smallest error on the Reuters corpus.
\paragraph{Non-negative garrote.} We consider the non-negative garrote which is a common approach to model order selection \cite{buhlmann2005boosting}. We evaluate NNMP, PWMP, and FCMP in the experiment described in~\cite{makalic2011logistic}, where the non-negative garrote is used to perform model order selection for logistic regression (i.e., for a non-quadratic objective function). We evaluated training and test accuracy on 100 random splits of the sonar dataset from the UCI machine learning repository. In Table~\ref{table:garroteText} we compare the median classification accuracy of our algorithms with that of the cyclic coordinate descent algorithm (NNG) from \cite{makalic2011logistic}. 

\vspace{2mm}
\begin{minipage}[h!]{\textwidth}
\begin{minipage}[f]{\textwidth}
\begin{minipage}{0.52\textwidth}
\centering
\footnotesize
\setlength\tabcolsep{1.2mm}
\begin{tabular}{ l | c c c c }
  algorithm &  \begin{tabular}[x]{@{}c@{}}Reuters \\$K=10$\end{tabular} &\begin{tabular}[x]{@{}c@{}}CBCL \\$K=10$\end{tabular} & \begin{tabular}[x]{@{}c@{}}CBCL \\$K=50$\end{tabular} & \begin{tabular}[x]{@{}c@{}}KNIX \\$K=10$\end{tabular}\\
  \hline
  mult & - & 2.4241e3 & 1.1405e3&2.4471e03\\  
    als & - & 2.73e3 & 3.84e3&2.7292e03\\  
  GCD & 5.9799e5 & 2.2372e3 & 806 &2.2372e03\\
  \textbf{PWMP} & \textbf{5.9591e5} & 2.2494e3 & 789.901 & 2.2494e03\\
  \textbf{FCMP} & {5.9762e5} & \textbf{2.2364e3} & \textbf{786.15} & \textbf{2.2364e03} \\
\end{tabular}
\vspace{-1.7mm}
\captionof{table}{Objective value for least-squares non-negative matrix factorization with rank $K$.}\label{table:nnmfText}
\end{minipage}\hspace{3mm}
\begin{minipage}{0.44\textwidth}
\footnotesize
\setlength\tabcolsep{1.2mm}
\vspace{3.8mm}
\center\begin{tabular}{ l | c | c }
   & training accuracy & test accuracy  \\
  \hline
  \textbf{NNMP} & 0.8345 $\pm$ 0.0242 & \textbf{0.7419} $\pm$ 0.0389 \\
  \textbf{PWMP} &\textbf{ 0.8379} $\pm$ 0.0240 & \textbf{0.7419} $\pm$ 0.0392  \\
  \textbf{FCMP} & 0.8345 $\pm$ 0.0238 & \textbf{0.7419} $\pm$ 0.0403 \\
  NNG & 0.8069 $\pm$ 0.0518 & 0.7258 $\pm$ 0.0602
\end{tabular}
\vspace{2mm}
\captionof{table}{Logistic Regression with non-negative Garrote, median $\pm$ std. dev.}\label{table:garroteText}
\end{minipage}
\end{minipage}
\end{minipage}
\section{Conclusion}
In this paper, we considered greedy algorithms for optimization over the convex cone, parametrized as the \textit{conic hull} of a generic atom set. We presented a novel formulation of NNMP along with a comprehensive convergence analysis. Furthermore, we introduced corrective variants with linear convergence guarantees, and verified this convergence rate in numerical applications. We believe that the generality of our novel analysis will be useful to design new, fast algorithms with convergence guarantees, and to study convergence of existing heuristics, in particular in the context of non-negative matrix and tensor factorization.

\clearpage
\newpage
{
\setlength{\bibsep}{1.5pt plus .5ex}
\bibliographystyle{plain}
\bibliography{bibliography}

\begin{thebibliography}{10}

\bibitem{behr2013mitie}
Jonas Behr, Andr{\'e} Kahles, Yi~Zhong, Vipin~T Sreedharan, Philipp Drewe, and
  Gunnar R{\"a}tsch.
\newblock Mitie: Simultaneous rna-seq-based transcript identification and
  quantification in multiple samples.
\newblock {\em Bioinformatics}, 29(20):2529--2538, 2013.

\bibitem{berry2007algorithms}
Michael~W Berry, Murray Browne, Amy~N Langville, V~Paul Pauca, and Robert~J
  Plemmons.
\newblock Algorithms and applications for approximate nonnegative matrix
  factorization.
\newblock {\em Computational statistics \& data analysis}, 52(1):155--173,
  2007.

\bibitem{bruckstein2008uniqueness}
Alfred~M Bruckstein, Michael Elad, and Michael Zibulevsky.
\newblock On the uniqueness of nonnegative sparse solutions to underdetermined
  systems of equations.
\newblock {\em IEEE Transactions on Information Theory}, 54(11):4813--4820,
  2008.

\bibitem{buhlmann2005boosting}
P~B{\"u}hlmann and B~Yu.
\newblock Boosting, model selection, lasso and nonnegative garrote.
\newblock Technical Report 127, Seminar f{\"u}r Statistik ETH Z{\"u}rich, 2005.

\bibitem{buhlmann2010boosting}
Peter B{\"u}hlmann and Bin Yu.
\newblock Boosting.
\newblock {\em Wiley Interdisciplinary Reviews: Computational Statistics},
  2(1):69--74, 2010.

\bibitem{chen1989orthogonal}
Sheng Chen, Stephen~A Billings, and Wan Luo.
\newblock Orthogonal least squares methods and their application to non-linear
  system identification.
\newblock {\em International Journal of control}, 50(5):1873--1896, 1989.

\bibitem{esser2013method}
Ernie Esser, Yifei Lou, and Jack Xin.
\newblock A method for finding structured sparse solutions to nonnegative least
  squares problems with applications.
\newblock {\em SIAM Journal on Imaging Sciences}, 6(4):2010--2046, 2013.

\bibitem{Frank1956bt}
M~Frank and P~Wolfe.
\newblock {An algorithm for quadratic programming}.
\newblock {\em Naval research logistics quarterly}, 1956.

\bibitem{gillis2016fast}
Nicolas Gillis and Robert Luce.
\newblock A fast gradient method for nonnegative sparse regression with self
  dictionary.
\newblock {\em arXiv preprint arXiv:1610.01349}, 2016.

\bibitem{grubb2011generalized}
Alexander Grubb and J~Andrew Bagnell.
\newblock Generalized boosting algorithms for convex optimization.
\newblock {\em arXiv preprint arXiv:1105.2054}, 2011.

\bibitem{guo2017efficient}
Xiawei Guo, Quanming Yao, and James~T Kwok.
\newblock Efficient sparse low-rank tensor completion using the {Frank-Wolfe}
  algorithm.
\newblock In {\em AAAI Conference on Artificial Intelligence}, 2017.

\bibitem{harchaoui2015conditional}
Zaid Harchaoui, Anatoli Juditsky, and Arkadi Nemirovski.
\newblock Conditional gradient algorithms for norm-regularized smooth convex
  optimization.
\newblock {\em Mathematical Programming}, 152(1-2):75--112, 2015.

\bibitem{hsieh2011fast}
Cho-Jui Hsieh and Inderjit~S Dhillon.
\newblock Fast coordinate descent methods with variable selection for
  non-negative matrix factorization.
\newblock In {\em Proceedings of the 17th ACM SIGKDD international conference
  on Knowledge discovery and data mining}, pages 1064--1072. ACM, 2011.

\bibitem{Jaggi:2013wg}
Martin Jaggi.
\newblock {Revisiting Frank-Wolfe: Projection-Free Sparse Convex Optimization}.
\newblock In {\em ICML 2013 - Proceedings of the 30th International Conference
  on Machine Learning}, 2013.

\bibitem{kim2014algorithms}
Jingu Kim, Yunlong He, and Haesun Park.
\newblock Algorithms for nonnegative matrix and tensor factorizations: A
  unified view based on block coordinate descent framework.
\newblock {\em Journal of Global Optimization}, 58(2):285--319, 2014.

\bibitem{kim2012fast}
Jingu Kim and Haesun Park.
\newblock Fast nonnegative tensor factorization with an active-set-like method.
\newblock In {\em High-Performance Scientific Computing}, pages 311--326.
  Springer, 2012.

\bibitem{LacosteJulien:2015wj}
Simon Lacoste-Julien and Martin Jaggi.
\newblock {On the Global Linear Convergence of Frank-Wolfe Optimization
  Variants}.
\newblock In {\em NIPS 2015}, pages 496--504, 2015.

\bibitem{Laue:2012wn}
S{\"o}ren Laue.
\newblock {A Hybrid Algorithm for Convex Semidefinite Optimization}.
\newblock In {\em ICML}, 2012.

\bibitem{lawson1995solving}
Charles~L Lawson and Richard~J Hanson.
\newblock {\em Solving least squares problems}, volume~15.
\newblock SIAM, 1995.

\bibitem{locatello2017unified}
Francesco Locatello, Rajiv Khanna, Michael Tschannen, and Martin Jaggi.
\newblock A unified optimization view on generalized matching pursuit and
  frank-wolfe.
\newblock In {\em Proc. International Conference on Artificial Intelligence and
  Statistics (AISTATS)}, 2017.

\bibitem{makalic2011logistic}
Enes Makalic and Daniel~F Schmidt.
\newblock Logistic regression with the nonnegative garrote.
\newblock In {\em Australasian Joint Conference on Artificial Intelligence},
  pages 82--91. Springer, 2011.

\bibitem{Mallat:1993gu}
St{\'e}phane Mallat and Zhifeng Zhang.
\newblock {Matching pursuits with time-frequency dictionaries}.
\newblock {\em IEEE Transactions on Signal Processing}, 41(12):3397--3415,
  1993.

\bibitem{meir2003introduction}
Ron Meir and Gunnar R{\"a}tsch.
\newblock An introduction to boosting and leveraging.
\newblock In {\em Advanced lectures on machine learning}, pages 118--183.
  Springer, 2003.

\bibitem{nguyen2014greedy}
Hao Nguyen and Guergana Petrova.
\newblock Greedy strategies for convex optimization.
\newblock {\em Calcolo}, pages 1--18, 2014.

\bibitem{ID52513}
Robert Peharz, Michael Stark, and Franz Pernkopf.
\newblock Sparse nonnegative matrix factorization using l0-constraints.
\newblock In IEEE, editor, {\em Proceedings of MLSP}, pages 83 -- 88, Aug 2010.

\bibitem{pena2015polytope}
Javier Pena and Daniel Rodriguez.
\newblock Polytope conditioning and linear convergence of the frank-wolfe
  algorithm.
\newblock {\em arXiv preprint arXiv:1512.06142}, 2015.

\bibitem{pena2016solving}
Javier Pena and Negar Soheili.
\newblock Solving conic systems via projection and rescaling.
\newblock {\em Mathematical Programming}, pages 1--25, 2016.

\bibitem{ratsch2001convergence}
Gunnar R{\"a}tsch, Sebastian Mika, Manfred~K Warmuth, et~al.
\newblock On the convergence of leveraging.
\newblock In {\em NIPS}, pages 487--494, 2001.

\bibitem{Sha:2002um}
F~Sha, LK~Saul, and Daniel~D Lee.
\newblock {Multiplicative updates for nonnegative quadratic programming in
  support vector machines}.
\newblock {\em Advances in Neural Information Processing Systems}, 15, 2002.

\bibitem{ShalevShwartz:2010wq}
Shai Shalev-Shwartz, Nathan Srebro, and Tong Zhang.
\newblock {Trading Accuracy for Sparsity in Optimization Problems with Sparsity
  Constraints}.
\newblock {\em SIAM Journal on Optimization}, 20:2807--2832, 2010.

\bibitem{shashua2005non}
Amnon Shashua and Tamir Hazan.
\newblock Non-negative tensor factorization with applications to statistics and
  computer vision.
\newblock In {\em Proceedings of the 22nd international conference on Machine
  learning}, pages 792--799. ACM, 2005.

\bibitem{Temlyakov:2013wf}
Vladimir Temlyakov.
\newblock {Chebushev Greedy Algorithm in convex optimization}.
\newblock {\em arXiv.org}, December 2013.

\bibitem{Temlyakov:2014eb}
Vladimir Temlyakov.
\newblock {Greedy algorithms in convex optimization on Banach spaces}.
\newblock In {\em 48th Asilomar Conference on Signals, Systems and Computers},
  pages 1331--1335. IEEE, 2014.

\bibitem{Temlyakov:2012vg}
VN~Temlyakov.
\newblock Greedy approximation in convex optimization.
\newblock {\em Constructive Approximation}, 41(2):269--296, 2015.

\bibitem{Tropp:2004gc}
Joel~A Tropp.
\newblock {Greed is good: algorithmic results for sparse approximation}.
\newblock {\em IEEE Transactions on Information Theory}, 50(10):2231--2242,
  2004.

\bibitem{wang2014matrixcompletion}
Zheng Wang, Ming jun Lai, Zhaosong Lu, Wei Fan, Hasan Davulcu, and Jieping Ye.
\newblock Rank-one matrix pursuit for matrix completion.
\newblock In {\em ICML}, pages 91--99, 2014.

\bibitem{Yaghoobi:2015ff}
Mehrdad Yaghoobi, Di~Wu, and Mike~E Davies.
\newblock Fast non-negative orthogonal matching pursuit.
\newblock {\em IEEE Signal Processing Letters}, 22(9):1229--1233, 2015.

\bibitem{Yang:2015wy}
Yuning Yang, Siamak Mehrkanoon, and Johan A~K Suykens.
\newblock {Higher order Matching Pursuit for Low Rank Tensor Learning}.
\newblock {\em arXiv.org}, March 2015.

\bibitem{yaogreedy}
Quanming Yao and James~T Kwok.
\newblock Greedy learning of generalized low-rank models.
\newblock In {\em IJCAI}, 2016.

\bibitem{yuan2013}
Xiao-Tong Yuan and Tong Zhang.
\newblock Truncated power method for sparse eigenvalue problems.
\newblock {\em J. Mach. Learn. Res.}, 14(1):899--925, April 2013.

\end{thebibliography}


\begin{thebibliography}{10}

\bibitem{anandkumar2014tensor}
Animashree Anandkumar, Rong Ge, Daniel~J Hsu, Sham~M Kakade, and Matus
  Telgarsky.
\newblock Tensor decompositions for learning latent variable models.
\newblock {\em Journal of Machine Learning Research}, 15(1):2773--2832, 2014.

\bibitem{araujo2001successive}
M{\'a}rio C{\'e}sar~Ugulino Ara{\'u}jo, Teresa Cristina~Bezerra Saldanha,
  Roberto Kawakami~Harrop Galvao, Takashi Yoneyama, Henrique~Caldas Chame, and
  Valeria Visani.
\newblock The successive projections algorithm for variable selection in
  spectroscopic multicomponent analysis.
\newblock {\em Chemometrics and Intelligent Laboratory Systems}, 57(2):65--73,
  2001.

\bibitem{berry2007algorithms}
Michael~W Berry, Murray Browne, Amy~N Langville, V~Paul Pauca, and Robert~J
  Plemmons.
\newblock Algorithms and applications for approximate nonnegative matrix
  factorization.
\newblock {\em Computational statistics \& data analysis}, 52(1):155--173,
  2007.

\bibitem{buhlmann2005boosting}
P~B{\"u}hlmann and B~Yu.
\newblock Boosting, model selection, lasso and nonnegative garrote.
\newblock Technical Report 127, Seminar f{\"u}r Statistik ETH Z{\"u}rich, 2005.

\bibitem{burger2003infinite}
Martin Burger.
\newblock Infinite-dimensional optimization and optimal design.
\newblock 2003.

\bibitem{cichocki2009fast}
Andrzej Cichocki and PHAN Anh-Huy.
\newblock Fast local algorithms for large scale nonnegative matrix and tensor
  factorizations.
\newblock {\em IEICE transactions on fundamentals of electronics,
  communications and computer sciences}, 92(3):708--721, 2009.

\bibitem{esser2013method}
Ernie Esser, Yifei Lou, and Jack Xin.
\newblock A method for finding structured sparse solutions to nonnegative least
  squares problems with applications.
\newblock {\em SIAM Journal on Imaging Sciences}, 6(4):2010--2046, 2013.

\bibitem{gillis2014successive}
Nicolas Gillis.
\newblock Successive nonnegative projection algorithm for robust nonnegative
  blind source separation.
\newblock {\em SIAM Journal on Imaging Sciences}, 7(2):1420--1450, 2014.

\bibitem{gillis2012accelerated}
Nicolas Gillis and Fran{\c{c}}ois Glineur.
\newblock Accelerated multiplicative updates and hierarchical als algorithms
  for nonnegative matrix factorization.
\newblock {\em Neural Computation}, 24(4):1085--1105, 2012.

\bibitem{gillis2015hierarchical}
Nicolas Gillis, Da~Kuang, and Haesun Park.
\newblock Hierarchical clustering of hyperspectral images using rank-two
  nonnegative matrix factorization.
\newblock {\em IEEE Transactions on Geoscience and Remote Sensing},
  53(4):2066--2078, 2015.

\bibitem{gillis2016fast}
Nicolas Gillis and Robert Luce.
\newblock A fast gradient method for nonnegative sparse regression with self
  dictionary.
\newblock {\em arXiv preprint arXiv:1610.01349}, 2016.

\bibitem{guo2017efficient}
Xiawei Guo, Quanming Yao, and James~T Kwok.
\newblock Efficient sparse low-rank tensor completion using the {Frank-Wolfe}
  algorithm.
\newblock In {\em AAAI Conference on Artificial Intelligence}, 2017.

\bibitem{hsieh2011fast}
Cho-Jui Hsieh and Inderjit~S Dhillon.
\newblock Fast coordinate descent methods with variable selection for
  non-negative matrix factorization.
\newblock In {\em Proceedings of the 17th ACM SIGKDD international conference
  on Knowledge discovery and data mining}, pages 1064--1072. ACM, 2011.

\bibitem{Jaggi:2013wg}
Martin Jaggi.
\newblock {Revisiting Frank-Wolfe: Projection-Free Sparse Convex Optimization}.
\newblock In {\em ICML 2013 - Proceedings of the 30th International Conference
  on Machine Learning}, 2013.

\bibitem{kim2007non}
Hyunsoo Kim, Haesun Park, and Lars Elden.
\newblock Non-negative tensor factorization based on alternating large-scale
  non-negativity-constrained least squares.
\newblock In {\em Bioinformatics and Bioengineering, 2007. BIBE 2007.
  Proceedings of the 7th IEEE International Conference on}, pages 1147--1151.
  IEEE, 2007.

\bibitem{kim2012fast}
Jingu Kim and Haesun Park.
\newblock Fast nonnegative tensor factorization with an active-set-like method.
\newblock In {\em High-Performance Scientific Computing}, pages 311--326.
  Springer, 2012.

\bibitem{kopriva2010nonlinear}
Ivica Kopriva and Andrzej Cichocki.
\newblock Nonlinear band expansion and 3d nonnegative tensor factorization for
  blind decomposition of magnetic resonance image of the brain.
\newblock In {\em International Conference on Latent Variable Analysis and
  Signal Separation}, pages 490--497. Springer, 2010.

\bibitem{kumar2013fast}
Abhishek Kumar, Vikas Sindhwani, and Prabhanjan Kambadur.
\newblock Fast conical hull algorithms for near-separable non-negative matrix
  factorization.
\newblock In {\em ICML (1)}, pages 231--239, 2013.

\bibitem{LacosteJulien:2013uea}
Simon Lacoste-Julien and Martin Jaggi.
\newblock {An Affine Invariant Linear Convergence Analysis for Frank-Wolfe
  Algorithms}.
\newblock In {\em NIPS 2013 Workshop on Greedy Algorithms, Frank-Wolfe and
  Friends}, December 2013.

\bibitem{LacosteJulien:2015wj}
Simon Lacoste-Julien and Martin Jaggi.
\newblock {On the Global Linear Convergence of Frank-Wolfe Optimization
  Variants}.
\newblock In {\em NIPS 2015}, pages 496--504, 2015.

\bibitem{Laue:2012wn}
S{\"o}ren Laue.
\newblock {A Hybrid Algorithm for Convex Semidefinite Optimization}.
\newblock In {\em ICML}, 2012.

\bibitem{lee2001algorithms}
Daniel~D Lee and H~Sebastian Seung.
\newblock Algorithms for non-negative matrix factorization.
\newblock In {\em Advances in neural information processing systems}, pages
  556--562, 2001.

\bibitem{locatello2017unified}
Francesco Locatello, Rajiv Khanna, Michael Tschannen, and Martin Jaggi.
\newblock A unified optimization view on generalized matching pursuit and
  frank-wolfe.
\newblock In {\em Proc. International Conference on Artificial Intelligence and
  Statistics (AISTATS)}, 2017.

\bibitem{makalic2011logistic}
Enes Makalic and Daniel~F Schmidt.
\newblock Logistic regression with the nonnegative garrote.
\newblock In {\em Australasian Joint Conference on Artificial Intelligence},
  pages 82--91. Springer, 2011.

\bibitem{nascimento2005vertex}
Jos{\'e}~MP Nascimento and Jos{\'e}~MB Dias.
\newblock Vertex component analysis: A fast algorithm to unmix hyperspectral
  data.
\newblock {\em IEEE transactions on Geoscience and Remote Sensing},
  43(4):898--910, 2005.

\bibitem{welling2001positive}
Max Welling and Markus Weber.
\newblock Positive tensor factorization.
\newblock {\em Pattern Recognition Letters}, 22(12):1255--1261, 2001.

\bibitem{Yaghoobi:2015ff}
Mehrdad Yaghoobi, Di~Wu, and Mike~E Davies.
\newblock Fast non-negative orthogonal matching pursuit.
\newblock {\em IEEE Signal Processing Letters}, 22(9):1229--1233, 2015.

\bibitem{yuan2013}
Xiao-Tong Yuan and Tong Zhang.
\newblock Truncated power method for sparse eigenvalue problems.
\newblock {\em J. Mach. Learn. Res.}, 14(1):899--925, April 2013.

\end{thebibliography}
}

\appendix
\clearpage

\section{Additional experiments} \label{app:experments}

\subsection{An illustrative experiment: Tightness of Theorem~\ref{thm:PWMPlinear}}
We now consider the setting depicted in Figure~\ref{pic:madw}.
\begin{figure}
\center\includegraphics[scale=0.6]{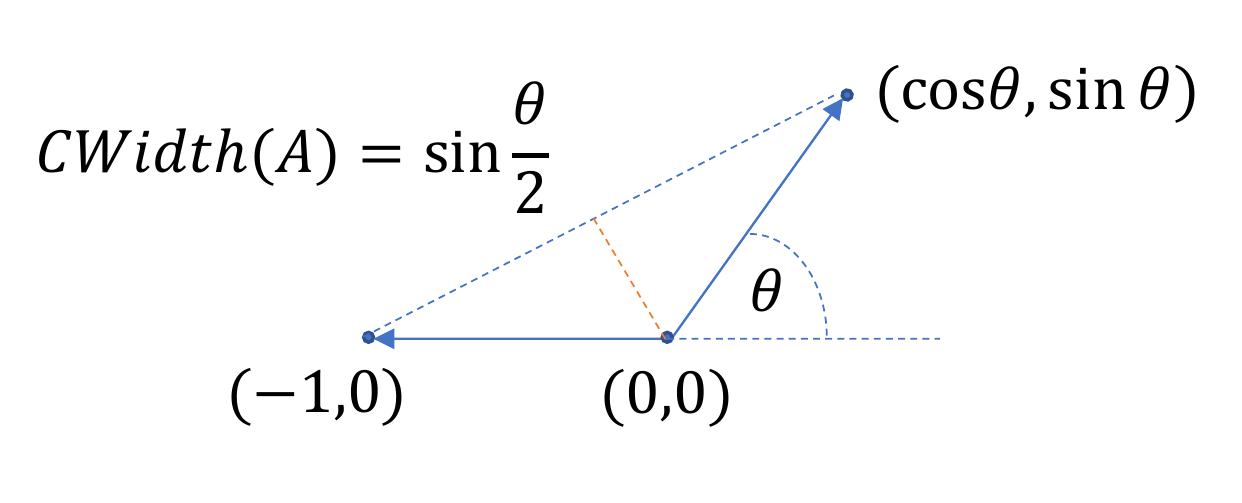}
\caption{$\cw$ for the set $
\cA := \{\cA_\theta \cup -\cA_\theta \}
\ \text{ where } \ 
\cA_\theta := \Big\{ {0\choose 0}, {-1\choose 0}, {\cos\theta \choose \sin\theta} \Big\}\vspace{-2mm}
$
with $\theta\in(0,\pi/2]$}\label{pic:madw}
\end{figure} 
We consider the set $\cA := \{\cA_\theta \cup -\cA_\theta \}
\ \text{ where } \ 
\cA_\theta := \Big\{ {0\choose 0}, {-1\choose 0}, {\cos\theta \choose \sin\theta} \Big\}
$
with $\theta\in(0,\pi/2)$. For this set $\cw$ can be computed in closed form as $\cw = \sin(\theta/2)$. We then perform 20 runs of Algorithm~\ref{algo:AMPPWMP} and report the ratio between the theoretical rate and the empirical one. The result is depicted in Figure~\ref{pic:tight}. There, we considered an iteration starting from the origin minimizing the distance function to 20 random points ${-\alpha_1 \choose \alpha_2}$ with $\alpha_i > 0$. The vertical bars shows minimal and maximal values.

\begin{figure}
\center\includegraphics[scale=0.4]{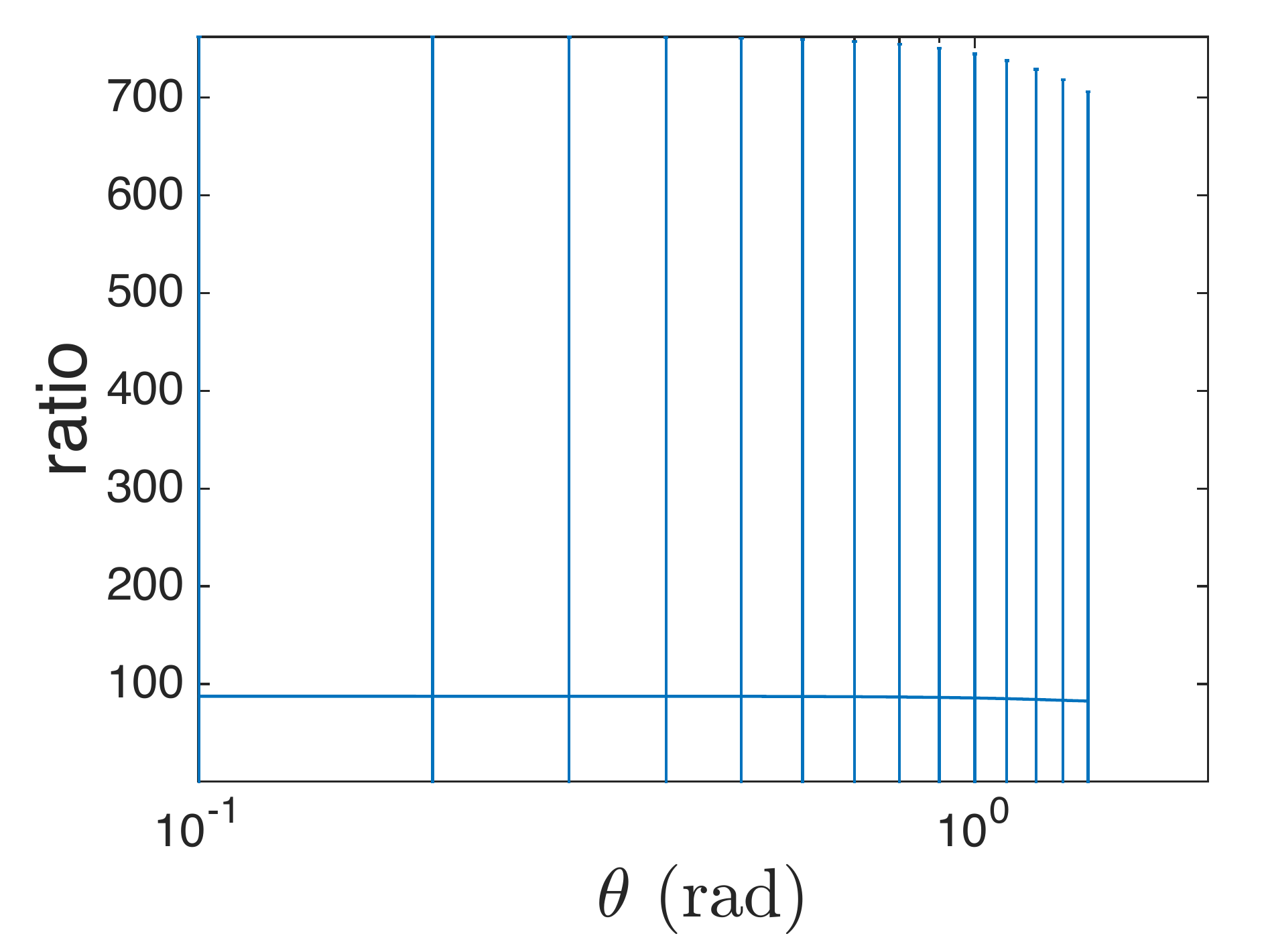}
\caption{Ratio of theoretical and empirical rate for $
\cA := \{\cA_\theta \cup -\cA_\theta \}
\ \text{ where } \ 
\cA_\theta := \Big\{ {0\choose 0}, {-1\choose 0}, {\cos\theta \choose \sin\theta} \Big\}
$
with $\theta\in(0,\pi/2)$ and 20 random target points ${-\alpha_1 \choose \alpha_2}$ with $\alpha_i > 0$. }\label{pic:tight}
\end{figure} 
\begin{figure}
\center\includegraphics[scale=0.2]{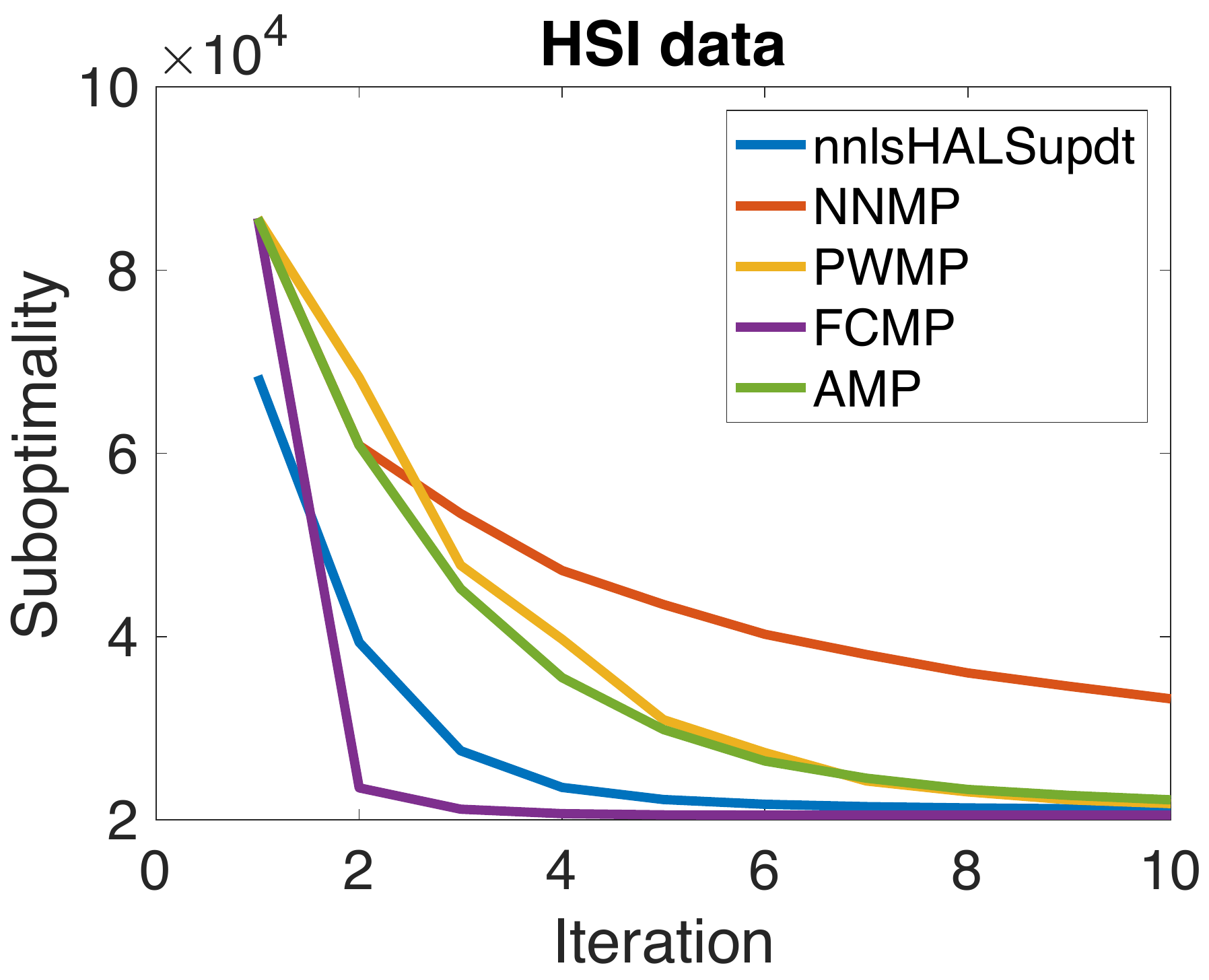}
\caption{Hyperspectral Imaging. We report suboptimality in a non-negative least squares task on real data, respectively.}\label{pic:comparisonhyp2} 
\end{figure} 

\vspace{-2mm}
\subsection{Real Data}
\vspace{-1mm}
\paragraph{Hyperspectral image unmixing.}
One of the classical applications of non-negative least squares are unmixing problems \citesup{esser2013method} such as hyperspectral image unmixing. Scalable unmixing approaches such as SPA \citesup{araujo2001successive} first extract a self-dictionary from a target image. Each pixel is then projected on the conic hull of the dictionary to estimate the abundance of each material. A standard technique is the hierarchical alternating least squares of \citesup{gillis2012accelerated} (nnlsHALSupdt). In Figure~\ref{pic:comparisonhyp2} (right), we compare the suboptimality of different methods as a function of the iteration. The dictionary is extracted from the undersampled Urban HSI Dataset\footnote{download at \url{http://bit.ly/fgnsr}} using SPA. This dataset contains 5,929 pixels, each associated with 162 hyperspectral features. The number of dictionary elements is 6, motivated by the fact that 6 different physical materials are depicted in this HSI data \citesup{gillis2016fast}. 
Therefore, FCMP converges after 6 iterations. For PWMP only $1.5\%$ of the iterations were bad steps on average for all dictionaries. Therefore, our corrective methods are proven to be competitive also on real data and the effect of the bad steps is negligible. We test other dictionaries for the Hyperspectral Imaging task. The result is depicted in Figure~\ref{pic:comparisonhyp}.
\begin{figure}
\center\includegraphics[scale=0.4]{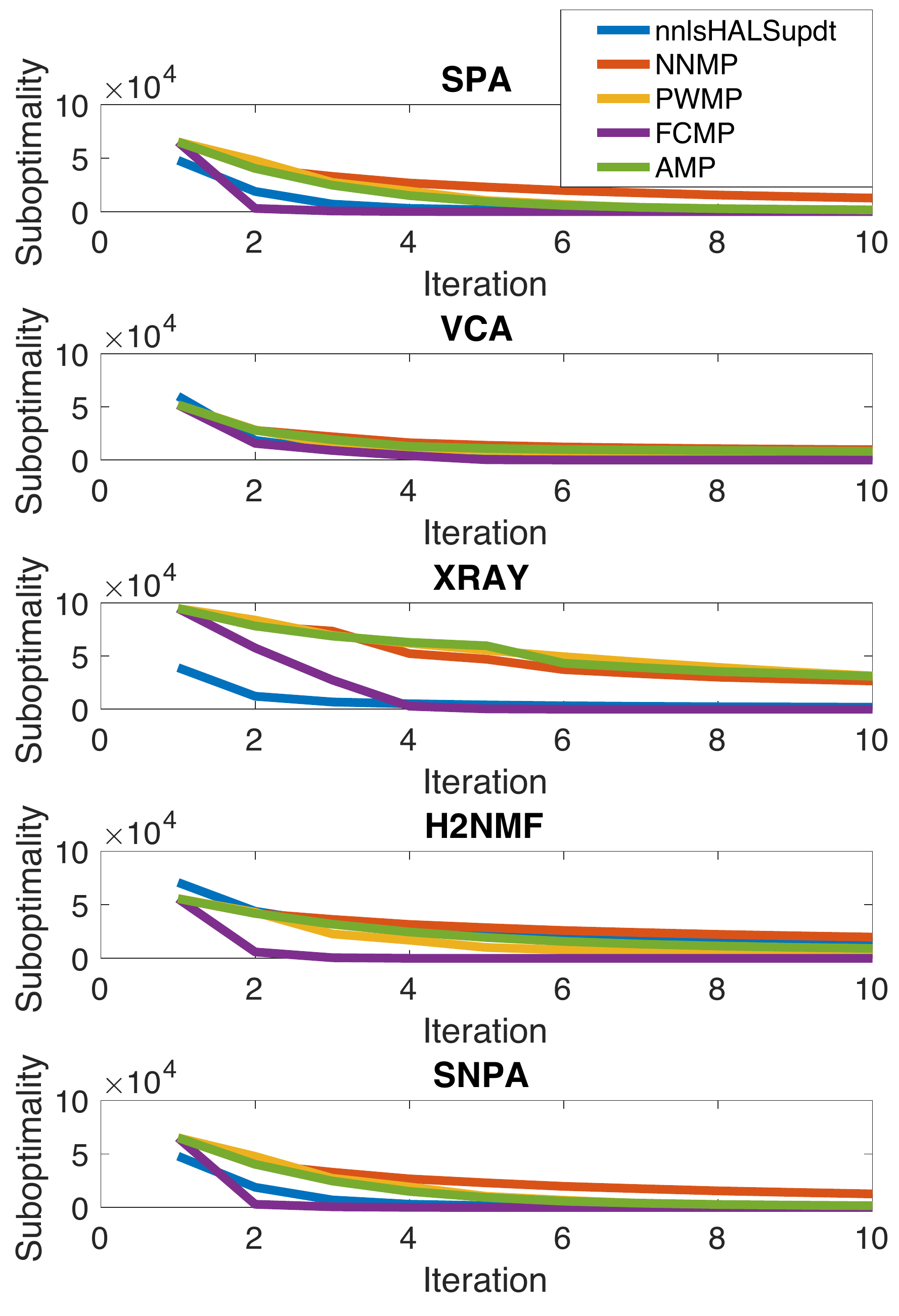}
\caption{SPA \protect\citesup{araujo2001successive}, VCA \protect\citesup{nascimento2005vertex}, XRAY \protect\citesup{kumar2013fast}, H2NMF \protect\citesup{gillis2015hierarchical}, SNPA \protect\citesup{gillis2014successive}} \label{pic:comparisonhyp} 
\end{figure} 

\vspace{-2mm}
\subparagraph{KL-divergence non-negative low-rank matrix factorization.}
The third task targets non-negative matrix factorization by minimization of the (non-least squares) KL-divergence-based objective function in Equation (3) in \citesup{hsieh2011fast}. 
\begin{figure}
\center\includegraphics[scale=0.24]{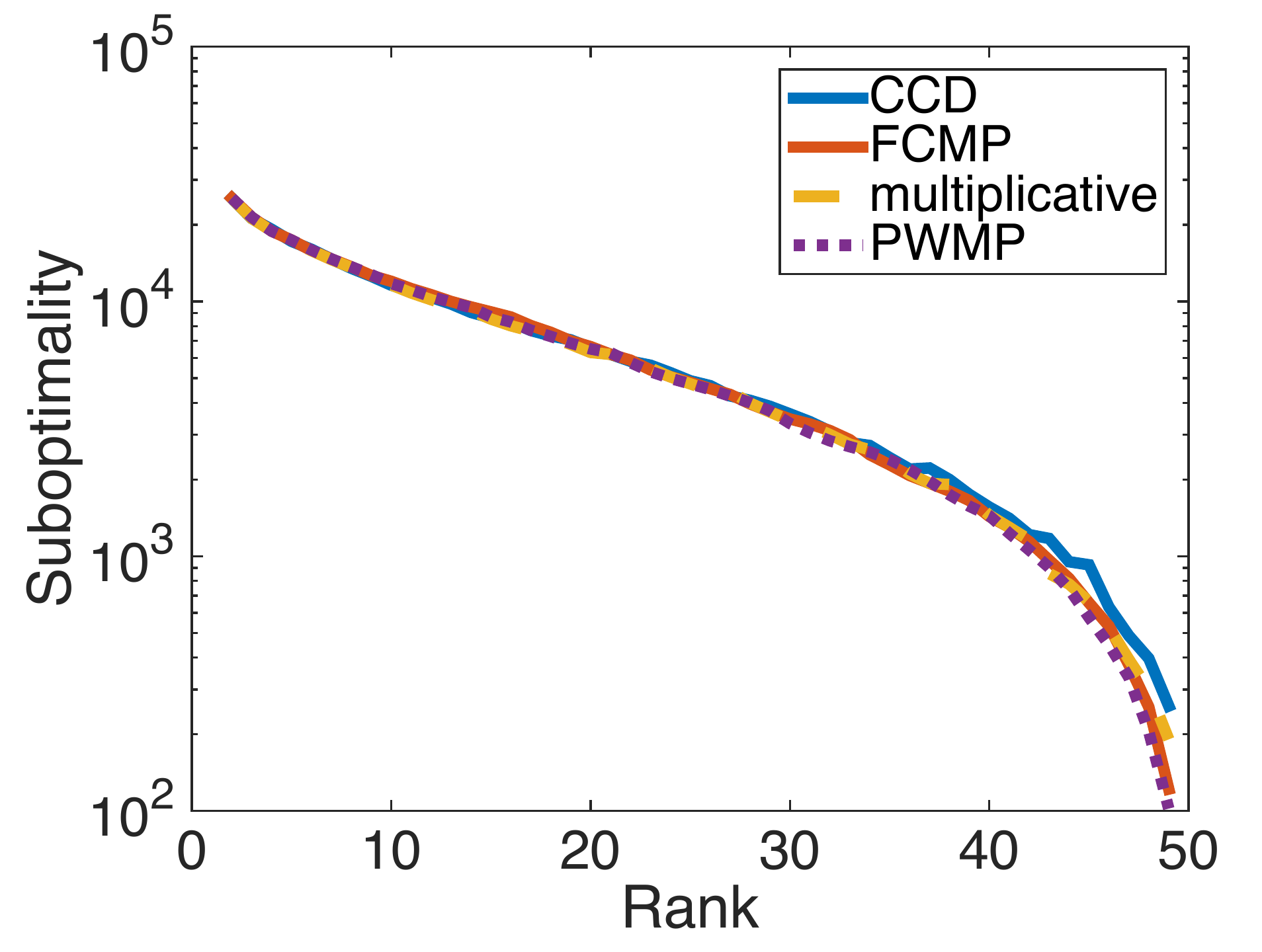}
\vspace{-3mm}
\caption{CBCL KL-Divergence} \label{pic:DKL} 
\vspace{-2mm}
\end{figure} 
We again use the CBCL face dataset and we compare FCMP (Variant 1) and PWMP against the multiplicative algorithm from \citesup{lee2001algorithms} (multiplicative) and the cyclic coordinate descent (CCD) from \citesup{hsieh2011fast}. We use the same approximate \lmo and parametrization of $\cA$ as for the least-squares non-negative matrix factorization, and we set the $L$ to 0.1. The atom correction was implemented using the CCD algorithm. In this experiment the use of Variant 0 of FCMP is crucial as it allows for a much easier update step. The objective value as a function of the rank is depicted in Figure~\ref{pic:DKL}. We note that all the algorithms yield comparable objective value up to rank 35. For higher rank, FCMP and PWMP achieve a slightly smaller objective value. 

\paragraph{Least Squares Non-negative Tensor Factorization.} \label{sec:tensorfactorization}
For this task we again use the KNIX dataset but now we arrange the scans to form a tensor of dimensionality $512\times 512 \times 24$. We compare against the alternating nonnegativity-constrained least squares with block principal pivoting \citesup{kim2012fast} (anlss-bpp) (which is also used in the FCMP and PWMP for the corrections), the active set method in \citesup{kim2007non} (anls-as), the hierarchical alternating least squares of \citesup{cichocki2009fast} (hals) and the multiplicative updating algorithm (multiplicative) of \citesup{welling2001positive}. The LMO for the tensor factorization is implemented with the tensor power method \citesup{anandkumar2014tensor}. The result is depicted in Table~\ref{table:nnmtf}.

\begin{table}
\center\begin{tabular}{ l | c  }
  algorithm & relative error \\
  \hline
    multiplicative & 0.2991 \\  
  hals & 0.2927 \\  
   anls-as & 0.2912 \\  
  anls-bpp & 0.2914 \\
  PWMP & 0.2913\\
  FCMP & \textbf{0.2909} \\
\end{tabular}
\caption{Non-negative tensor factorization on the KNIX dataset with rank 20}\label{table:nnmtf}
\end{table}

In this section we showed that the merit of our algorithms is not limited to their theoretical properties. Indeed, our algorithms are competitive with several approaches and can be successfully used in a manifold of different tasks and datasets while not being tailored to any specific cost function. 
\section{Affine Invariant Algorithms and Rates}
\label{sec:AffInv}
In this section, present affine invariant versions of all presented algorithms, along with sub-linear and linear convergence guarantees. An optimization method is called \emph{affine invariant} if it is invariant
under linear or affine transformations of the input problem: If one chooses any
re-parameterization of the domain~$\domain$ by a \emph{surjective} linear or affine map $\bM:\hat\domain\rightarrow\domain$, then the ``old'' and ``new''
optimization problems $\min_{\bx\in\domain}f(\bx)$ and
$\min_{\hat\bx\in\hat\domain}\hat f(\hat\bx)$ for $\hat f(\hat\bx):=f(\bM\hat\bx)$
look the same to the algorithm. We still require the set $\domain$ to contain the origin. In the following, we assume that after the transformation the origin is still on the border of $\conv(\domain)$. If the origin is contained in the relative interior of $\conv(\domain)$ we recover the existing affine invariant rates of \citesup{locatello2017unified}.

\subsection{Affine invariant non-negative MP}
To define an affine invariant upper bound on the objective function $f$, we use a variation of the affine invariant definition of the \emph{curvature constant} from \citesup{Jaggi:2013wg}, adapted for MP in \citesup{locatello2017unified}:
\begin{equation}\label{eq:CfMP}
\CfMP := \sup_{\substack{\bs\in\cA,\, \bx \in \conv(\cA) \\ \gamma \in [0,1]\\ \by = \bx + \gamma \bs}} \frac{2}{\gamma^2} D(\by,\bx),
\end{equation}
where for cleaner exposition, we have used the shorthand notation $D(\by,\bx)$ to denote the difference of $f(\by)$ and its linear approximation at $\bx$
\begin{equation*}
D(\by,\bx) :=f(\by) - f(\bx)- \langle \by -\bx, \nabla f(\bx)\rangle.
\end{equation*}
Bounded curvature $\Cf$ closely corresponds to smoothness of the objective~$f$. More precisely, if~$\nabla f$ is $L$-Lipschitz continuous on $\conv(\cA)$ with
respect to some arbitrary chosen norm $\norm{.}$, i.e.
$\norm{\nabla f(\bx) - \nabla f(\by)}_* \leq L \norm{\bx-\by}$, where $\|.\|_*$ is the dual norm of $\|.\|$, then
\begin{equation} \label{eq:CfBound}
\Cf \le L \radius_{\norm{.}}(\cA)^2  \ ,
\end{equation}
where $\radius_{\norm{.}}(.)$ denotes the $\norm{.}$-radius, see Lemma
15 in \citesup{locatello2017unified}. The curvature constant $\Cf$ is affine invariant as it does not depend on any
norm. It combines the complexity of the domain $\conv(\cA)$ and the curvature of the objective function~$f$ into a single quantity. 
Throughout this section, we assume availability of a finite constant $\rho>0$ upper-bounding the atomic norms~$\|.\|_\cA$ of the optimum $\bx^\star$, as well as the iterate sequence $(\bx_t)_{t=0}^T$ until the current iteration, as defined in~\eqref{eq:rho}.
We now present the affine invariant version of the non-negative MP algorithm (Algorithm \ref{algo:NNMP}) in Algorithm \ref{algo:NNMPaffine}. 
The algorithm uses the curvature constant $\CfMPr$ over the re-scaled set $\rho \conv(\cA\cup -\cA)$, rather than $\conv(\cA\cup-\cA)$.
\vspace{-2mm}

\begin{algorithm}[h]
\caption{Affine Invariant Non-Negative Matching Pursuit}
\label{algo:NNMPaffine}
\begin{algorithmic}
  \STATE Same as Algorithm~\ref{algo:NNMP} except:
  \STATE 5: \qquad $\gamma := \frac{\langle -\nabla f(\bx_t), \rho^2\bz_t   \rangle}{\CfMPr}$
\end{algorithmic}
\end{algorithm}

A sub-linear convergence guarantee for Algorithm \ref{algo:NNMPaffine} is presented in the following theorem.

\begin{theorem}\label{thm:sublinearNNMPAffineInvariant}
Let $\cA \subset \cH$ be a bounded set with $\0\in\cA$, $\rho := \max\left\lbrace \|\bx^\star\|_{\cA}, \|\bx_{0}\|_{\cA},\ldots,\|\bx_T\|_{\cA}\right\rbrace < \infty$. Assume $f$ has smoothness constant $\CfMPr$.
Then, Algorithm~\ref{algo:NNMPaffine} converges for $t \geq 0$ as 
\[
f(\bx_t) - f(\bx^\star)\leq \frac{4\left(\frac2\delta \CfMPr+\varepsilon_0\right)}{\delta t+4} ,
\]
where $\delta \in (0,1]$ is the relative accuracy parameter of the employed approximate \lmo~\eqref{eq:inexactLMOMP}.
\end{theorem}

Exact knowledge of $\CfMPr$ is not required; the same theorem also holds if any upper bound on $\CfMPr$ is used in the algorithm and resulting rate instead. 

\subsection{Affine invariant corrective MP}
An affine invariant version of AMP and PWMP, Algorithm~\ref{algo:AMPPWMP}, is presented in Algorithm~\ref{algo:AMPPWaffine}. Note that Variant 1 of the fully corrective non-negative MP in Algorithm~\ref{algo:FCNNMP} is already affine invariant as it does not rely on any norm. Note that sublinear convergence is guaranteed with the rate indicated by Theorem~\ref{thm:sublinearNNMPAffineInvariant} since each step of the affine invariant FCMP yields at least as much improvement as the affine invariant NNMP, Algorithm~\ref{algo:NNMPaffine}. 

Since the update step in Algorithm~\ref{algo:NNMPaffine} and the resulting upper bound on the progress in objective, based on the curvature constant \eqref{eq:CfBound}, we used in the proof of Theorem~\ref{thm:sublinearNNMPAffineInvariant} are not enough to ensure linear convergence, we use a different notion of curvature based on \citesup{LacosteJulien:2013uea}.
\begin{align*}
\CfMPPW = \sup_{\substack{\bs\in\cA,\bx\in\conv(\cA)\\\bv\in\cS\\\gamma\in [0,1]\\\by = \bx+\gamma(\bs-\bv)}} \frac{2}{\gamma^2}D(\by,\bx).
\end{align*}
\begin{algorithm}[h]
\caption{Affine invariant AMP and PWMP}
\label{algo:AMPPWaffine}
\begin{algorithmic}
  \STATE same as Algorithm~\ref{algo:AMPPWMP} except for:
  \STATE 5: \qquad $\gamma := \frac{\langle -\nabla f(\bx_t),\rho^2\bd_t\rangle}{\CfMPrPW}$
\end{algorithmic}
\end{algorithm}
The following positive step size quantity relates the dual certificate value of the descent direction $\bx^\star - \bx$ with the MP selected atom, 
\begin{equation}
\label{def:stepsizeGamma}
\gamma(\bx,\bx^\star) := \frac{ \langle-\nabla f(\bx), \bx^\star - \bx \rangle}{\langle-\nabla f(\bx),  \bs(\bx)-\bv(\bx) \rangle} \ ,
\end{equation}
for $\bs(\bx) := \argmin_{\bs\in\cA}\ \langle \nabla f(\bx), \bs\rangle$ and $\bv(\bx) := \min_{\cS\in\cS_\bx}\argmax_{\bs\in\cS}\ \langle \nabla f(\bx), \bs\rangle$
where $\cS\in\cS_{\bx}$ is the active set.
We now define the affine invariant surrogate of strong convexity.
\begin{align*}
\mufr := \inf_{\bx \in \conv(\rho\cA)} \inf_{\substack{\bx^\star \in \conv(\rho\cA)\\ \langle\nabla f (\bx), \bx^\star - \bx \rangle < 0 }} \frac{2}{\gamma(\bx,\bx^\star)} D(\bx^\star,\bx).\vspace{-1mm}
\end{align*}

\begin{theorem}
\label{thm:LinearMPAffineInvariant}
Let $\cA \subset \cH$ be a bounded set containing the origin
and let the objective function $f \colon \cH \to \R$ have smoothness constant $\CfMPrPW$ and strong convexity constant $\mufr$

Then, the suboptimality of the iterates of Algorithm~\ref{algo:AMPPWMP} and \ref{algo:FCNNMP} decreases geometrically at each step in which $\gamma < \alpha_{\bv_t}$ (henceforth referred to as ``good steps'') as:
\begin{equation} 
\varepsilon_{t+1}
\leq \left(1- \beta \right)\varepsilon_{t},
\end{equation}
where $\beta := \delta^2\frac{\mufr}{\CfMPrPW}\in (0,1]$, $\varepsilon_t := f(\bx_t) - f(\bx^\star)$ is the suboptimality at step $t$ and $\delta \in (0,1]$ is the relative accuracy parameter of the employed approximate \lmo  (Equation~\eqref{eq:inexactLMOMP}). For AMP (Algorithm~\ref{algo:AMPPWMP}), $\beta^{\text{AMP}} = \beta/2$.
 If $\mufr = 0$ Algorithm~\ref{algo:AMPPWMP} converges with rate $O(1/k(t))$ where $k(t)$ is the number of ``good steps'' up to iteration t.
\end{theorem}

\section{Proof of Lemma~\ref{lemma:originOpt}} 
If $\tilde\bx \in \cone(\cA)$ and $\nabla f(\tilde\bx)\not\in T_\cA$ then $\tilde\bx$ is a solution to $\min_{\bx\in\cone(\cA)}f(\bx)$.
\begin{proof}
We will prove this lemma by contradiction assuming that $\bx^\star\neq \tilde\bx$ and $\nabla f(\tilde\bx)\not\in T_\cA$. Now, by convexity of $f$ we have that:
\begin{align*}
f(\bx^\star)\geq f(\tilde\bx) + \langle \nabla f(\tilde\bx),\bx^\star-\tilde\bx\rangle
\end{align*}
Since $\bx^\star\neq \tilde\bx$ we have also that $f(\bx^\star)<f(\tilde\bx)$. Therefore:
\begin{align*}
0< f(\tilde\bx) - f(\bx^\star)\leq  \langle - \nabla f(\tilde\bx),\bx^\star-\tilde\bx\rangle
\end{align*}
which we rewrite as $\langle  \nabla f(\tilde\bx),\bx^\star\rangle+\langle  \nabla f(\tilde\bx),-\tilde\bx\rangle<0$. Now we note that
by the assumption that $\nabla f(\tilde\bx)\not\in T_\cA$ we have that both these inner products are non negative which is absurd.
To draw this conclusion note that $\bx^\star\in \cone(\cA)$ we have that $\bx^\star = \sum_i \alpha_i \bz_i$ where $\bz_i\in\cA$ and $\alpha_i\geq0 \ \forall \ i $. 
\end{proof}
\section{Sublinear Rates} \label{sec:sublinpf}
\begin{reptheorem}{thm:NNMPsublinear}
Let $\cA \subset \cH$ be a bounded set with $\0\in\cA$, $\rho := \max\left\lbrace \|\bx^\star\|_{\cA}, \|\bx_{0}\|_{\cA},\ldots,\|\bx_T\|_{\cA},\right\rbrace $ and $f$ be $L$-smooth over $\rho\conv(\cA\cup-\cA)$.
Then, Algorithms~\ref{algo:NNMP} and \ref{algo:FCNNMP} with $\bx_0 = \0$ converges for $t \geq 0$ as 
\[
f(\bx_t) - f(\bx^\star)\leq \frac{4\left(\frac2\delta L\rho^2\radius(\cA)^2+\varepsilon_0\right)}{\delta t+4} ,
\]
where $\delta \in (0,1]$ is the relative accuracy parameter of the employed approximate \lmo~\eqref{eq:inexactLMOMP}.
\begin{proof}
We separately prove the convergence for the two algorithms.

\paragraph{non-negative MP:}
Recall that $\tilde{\bz}_t$ is the atom returned by the inexact \lmo after the comparison with $-\frac{\bx_t}{\|\bx_t\|_\cA}$ at the current iteration $t$. 
We distinguish the two cases in which $\tilde{\bz}_t\neq -\frac{\bx_t}{\|\bx_t\|_\cA}$ (\textbf{case A}) and $\tilde{\bz}_t= -\frac{\bx_t}{\|\bx_t\|_\cA}$ (\textbf{case B}).
Let us call $\bar{\cA}:= \cA \cup \left\lbrace-\frac{\bx_t}{\|\bx_t\|_\cA}\right\rbrace$. Note that $\radius(\bar\cA)=\radius(\cA)$.

Recall that in the Algorithm the step size $\gamma$ is computed at each iteration via line search minimizing the quadratic upper bound on $f$ and no further clipping is made. The reason being that $f$ is convex, therefore, for $t>0$ we have $f(\bx_t)\leq f(\0)$. Hence the minimum of $f$ over the line between $\bx_t$ and the origin must lie between these two points making clipping unnecessary. 

We start by upper bounding $f$ on $\rho \conv(\bar\cA)$ using smoothness as follows:
\begin{eqnarray}\label{eq:quadBuondSub}
f(\bx_{t+1}) &\leq &  \min_{\gamma\in\mathbb{R}_{\geq 0}}g_{\bx_t}(\bx_t+\gamma\tilde\bz_t) \nonumber \\
&= &  \min_{\gamma\in[0,1]}g_{\bx_t}(\bx_t+\gamma\rho\tilde\bz_t) \nonumber \\
&\leq &  \min_{\gamma\in[0,1]}f(\bx_t)+\gamma\langle\nabla f(\bx_t),\rho\tilde\bz_t \rangle \nonumber\\ &+& \frac{L}{2}\gamma^2\rho^2\|\tilde\bz_t\|^2\nonumber \\
&\leq &  \min_{\gamma\in[0,1]}f(\bx_t)+\gamma\langle\nabla f(\bx_t),\rho\tilde\bz_t \rangle \nonumber\\ &+& \frac{L}{2}\gamma^2\rho^2\radius(\cA)^2\nonumber \\
\end{eqnarray}
We now treat separately the linear term for \textbf{case~A} and \textbf{case~B}.
\paragraph{case A:}
We start from the definition of inexact \lmo (Equation~\eqref{eq:inexactLMOMP}). 
We then have:
\begin{align*}
\langle \nabla f(\bx_t), \tilde{\bz}_t\rangle \leq \delta\langle \nabla f(\bx_t), \bz_t\rangle 
\end{align*}
where $\bz_t$ is the true minimizer of the linear problem (\lmo). In other words, it holds that $\langle \nabla f(\bx_t), \bz_t\rangle\leq \langle \nabla f(\bx_t), \bz\rangle \ \forall \ \bz \in \conv(\bar{\cA})$ due to the $\argmin$ in line 4 of Algorithm~\ref{algo:NNMP}.
Therefore, since $\bx^\star\in\rho\conv(\bar\cA)$ it holds that:
\begin{align*}
\langle\nabla f(\bx_t),-\rho\bz_t\rangle\geq \langle\nabla f(\bx_t),-\bx^\star\rangle.
\end{align*}
Using the same argument, since $-\frac{\bx_t}{\|\bx_t\|_\cA}\in \bar\cA$ and $\rho\geq \|\bx_t\|_\cA$, we have that $-\bx_t\in\rho\conv(\bar\cA)$. Therefore:
\begin{align*}
\langle\nabla f(\bx_t),-\rho\bz_t\rangle\geq \langle\nabla f(\bx_t),\bx_t\rangle.
\end{align*}
We can now bound the linear term of in \eqref{eq:quadBuondSub} as:
\begin{align*}
 \langle \nabla f(\bx_{t}), - 2\frac{\rho}{\delta} \tilde{\bz}_t\rangle &= \langle \nabla f(\bx_{t}), - \frac{\rho}{\delta} \tilde{\bz}_t\rangle+\langle\nabla f(\bx_{t}), - \frac{\rho}{\delta} \tilde{\bz}_t\rangle\nonumber\\
 &\geq \langle \nabla f(\bx_{t}), - \rho\bz_t\rangle+\langle\nabla f(\bx_{t}), -\rho\bz_t\rangle\\
 &\geq \langle\nabla f(\bx_{t}),\bx_{t} - \bx^\star\rangle\\
 &\geq f(\bx_{t}) -f(\bx^\star)=: \varepsilon_{t}
 \end{align*}
 where in the inequalities we used the the inexact oracle definition (see Section \ref{sec:approxlmo}), 
the fact that both $-\bx_t$ and $\bx^\star\in\rho\conv(\bar\cA)$ 
and convexity respectively.

\paragraph{case B:}
Using line 4 of Algorithm~\ref{algo:NNMP} along with the inexact oracle definition we obtain:
\begin{align*}
\langle \nabla f(\bx_t), -\frac{\bx_t}{\|\bx_t\|_\cA}\rangle\leq\delta\min_{\bz\in\cA}\langle \nabla f(\bx_t),\bz\rangle.
\end{align*}
 Therefore, since $\bx^\star\in\rho\conv(\cA)$ we can write: 
\begin{align*}
\langle \nabla f(\bx_t), -\frac{\rho}{\delta}\frac{\bx_t}{\|\bx_t\|_\cA}\rangle&\leq\min_{\bz\in\cA}\langle \nabla f(\bx_t),\rho\bz\rangle\\
&\leq\langle \nabla f(\bx_t),\bx^\star\rangle
\end{align*}
We also have $\langle \nabla f(\bx_t), -\bx_t\rangle\leq0$ and $\frac{\rho}{\delta \|\bx_t\|_\cA}>1$, which yields:
\begin{align*}
\langle \nabla f(\bx_t), -\bx_t\rangle\geq\langle \nabla f(\bx_t), -\frac{\rho}{\delta }\frac{\bx_t}{\|\bx_t\|_\cA}\rangle
\end{align*}
Putting these inequalites together we obtain:
\begin{align*}
\langle\nabla f(\bx_t),\frac{2}{\delta}\rho\frac{\bx_t}{\|\bx_t\|_\cA}\rangle &\geq \langle\nabla f(\bx_t),\bx_t\rangle+\max_{\bz\in\cA}\langle\nabla f(\bx_t),-\rho\bz\rangle\\
&\geq \langle\nabla f(\bx_t),\bx_t\rangle - \langle\nabla f(\bx_t),\bx^\star\rangle\\
&\geq \varepsilon_t
\end{align*}

\paragraph{combining A and B}
By combining case A and case B we obtain:
\begin{align*}
 f(\bx_{t+1})\leq f(\bx_t) +  \min_{\gamma\in[0,1]} \left\lbrace - \frac{\delta}{2} \gamma \varepsilon_t + \frac{\gamma^2}{2} L\rho^2\radius(\cA)^2 \right\rbrace
\end{align*}

Now, subtracting $f(\bx^\star)$ from both sides and setting $C:= L\rho^2\radius(\cA)^2 $, we get
\begin{eqnarray*}
 \varepsilon_{t+1} &\leq \varepsilon_{t} + \min_{\gamma\in[0,1]}\left\lbrace - \frac{\delta}{2} \gamma \varepsilon_{t} + \frac{\gamma^2}{2}C\right\rbrace\\
 & \leq \varepsilon_{t} - \frac{2}{\delta' t+2}\delta' \varepsilon_{t} + \frac{1}{2}\left(\frac{2}{\delta' t+2}\right)^2 C,
\end{eqnarray*}
where we set $\delta' := \delta/2$ and used $\gamma = \frac{2}{\delta' t+2} \in [0,1]$ to obtain the second inequality. Finally, we show by induction
 \begin{equation*}
 \varepsilon_t \leq \frac{4\left(\frac{2}{\delta} C + \varepsilon_0\right)}{t+4} = 2\frac{\left(\frac{1}{\delta'} C + \varepsilon_0\right)}{\delta' t+2}
 \end{equation*}
for $t \geq 0$.

When $t=0$ we get $\varepsilon_0\leq \left(\frac{1}{\delta'}C+\varepsilon_0\right)$. Therefore, the base case holds. We now prove the induction step assuming $\varepsilon_t \leq \tfrac{2\left(\frac{1}{\delta'}C+\varepsilon_0\right)}{\delta' t+2}$ as :
\begin{align*}
\varepsilon_{t+1} &\leq \left(1-\tfrac{2\delta'}{\delta' t + 2}\right)\varepsilon_{t} + \tfrac12 C \left(\tfrac{2}{\delta' t+2}\right)^2\\
&\leq \left(1-\tfrac{2\delta'}{\delta' t + 2}\right)\tfrac{2\left(\frac{1}{\delta'}C+\varepsilon_0\right)}{\delta' t+2} \\
&\quad+ \tfrac{1}{2}\left(\tfrac{2}{\delta' t+2}\right)^2C + \tfrac{2}{(\delta' t+2)^2}\delta'\varepsilon_0\\
&= \tfrac{2\left(\frac{1}{\delta'}C+\varepsilon_0\right)}{\delta' t+2}\left(1-\tfrac{2\delta'}{\delta' t +2}+\tfrac{\delta'}{\delta' t +2}\right)\\
&\leq \tfrac{2\left(\frac{1}{\delta'}C+\varepsilon_0\right)}{\delta'(t+1)+2}.
\end{align*}
Remembering that we set $C = L\rho^2\radius(\cA)^2$ concludes the proof.

\paragraph{Fully Corrective non-negative MP:}
The proof is trivial considering that:
\begin{align}
f(\bx_{t+1}) &=\min_{\bx\in\cone(\cS\cup \bs(\cA,\br))}f(\bx)\\
&\leq \min_{\bx\in\cone(\cS\cup \bs(\cA,\br))}g_{\bx_t}(\bx)\label{eq:subFCproof}\\
&\leq \min_{\gamma\in \bbR_{\geq0}}g_{\bx_t}(\bx_t+\gamma\tilde\bz_t)\label{eq:FCtoMPproof}
\end{align}
where $\tilde\bz_t\in\vA\cup\left\lbrace\frac{-\bx_t}{\|\bx_t\|_{\cA}}\right\rbrace$ as the search space in Equation~\eqref{eq:subFCproof} strictly contain the one in Equation~\eqref{eq:FCtoMPproof}. Equation~\eqref{eq:FCtoMPproof} is also the beginning of the proof of the sublinear rate for NNMP which then concludes the proof.
\end{proof}
\end{reptheorem}

\section{Linear Rate}
\begin{reptheorem}{thm:PWMPlinear}
Let $\cA \subset \cH$ be a bounded set containing the origin
and let the objective function $f \colon \cH \to \R$ be both $L$-smooth and $\mu$-strongly convex over $\rho \conv(\cA\cup-\cA)$.

Then, the suboptimality of the iterates of Algorithm~\ref{algo:AMPPWMP} decreases geometrically at each step in which $\gamma < \alpha_{\bv_t}$ (henceforth referred to as ``good steps'') as:
\begin{equation}
\varepsilon_{t+1}
\leq \left(1- \beta \right)\varepsilon_{t},
\end{equation}
where $\beta := \delta^2\frac{\mu \cw^2}{L\diam(\cA)^2}\in (0,1]$, $\varepsilon_t := f(\bx_t) - f(\bx^\star)$ is the suboptimality at step $t$ and $\delta \in (0,1]$ is the relative accuracy parameter of the employed approximate \lmo  (Equation~\eqref{eq:inexactLMOMP}). For AMP (Algorithm~\ref{algo:AMPPWMP}), $\beta^{\text{AMP}} = \beta/2$. If $\mu = 0$ Algorithm~\ref{algo:AMPPWMP} converges with rate $O(1/k(t))$ where $k(t)$ is the number of ``good steps'' up to iteration t.
\begin{proof}
Let us consider the case of PWMP.

Consider the atoms $\tilde{\bz}_t\in\cA$ and $\tilde{\bv}_t\in\cS$ selected by the \lmo at iteration $t$.
Due to the smoothness property of $f$ it holds that: 
\begin{align*}
f(\bx_{t+1})&\leq \min_{\gamma\in \mathbb{R}} f(\bx_{t}) + \gamma\langle\nabla f(\bx_{t}), \tilde{\bz}_t - \tilde{\bv}_t\rangle \\&+ \frac{L}{2}\gamma^2\|\tilde{\bz}_t - \tilde{\bv}_t\|^2.
\end{align*}
for a good step (i.e. $\gamma < \alpha_{\bv_t}$). Note that this also holds for variant 0 of Algorithm~\ref{algo:FCNNMP}.

We minimize the upper bound with respect to $\gamma$ setting $\gamma = -\frac{1}{L}\langle\nabla f(\bx_{t}),\frac{\tilde{\bz}_t - \tilde{\bv}_t}{\|\tilde{\bz}_t - \tilde{\bv}_t\|^2}\rangle$ . Subtracting $f(\bx^\star)$ from both sides and replacing the optimal $\gamma$ yields:
\begin{equation}\label{step:linearRateMPsmoothfast}
\varepsilon_{t+1}\leq\varepsilon_{t}-\frac{1}{2L}\left\langle\nabla f(\bx_{t}),\frac{\tilde{\bz}_t - \tilde{\bv}_t}{\|\tilde{\bz}_t-\tilde{\bv}_t\|}\right\rangle^2
\end{equation}
Now writing the definition of strong convexity, we have the following inequality holding for all  $\gamma\in\R$:
\begin{align*}
f(\bx_{t}+\gamma(\bx^\star-\bx_{t}))\geq f(\bx_{t})+\gamma\langle\nabla f(\bx_{t}),\bx^\star-\bx_{t}\rangle+\\\quad\gamma^2\frac{\mu}{2}\|\bx^\star-\bx_{t}\|^2
\end{align*}
We now fix $\gamma=1$ in the LHS and minimize with respect to $\gamma$ in the RHS:
\begin{align*}
\varepsilon_{t}\leq \frac{1}{2\mu} \left\langle \nabla f(\bx_{t}), \frac{\bx^\star-\bx_{t}}{\|\bx^\star-\bx_{t}\|} \right\rangle^2
\end{align*}
Combining this with \eqref{step:linearRateMPsmoothfast} yields: 
\begin{align}\label{step:linearRateCombinedfast}
\varepsilon_{t}-\varepsilon_{t+1}\geq \frac{\mu}{L}\frac{\big\langle \nabla f(\bx_{t}),\frac{\tilde{\bz}_t-\tilde{\bv}_t}{\|\tilde{\bz}_t-\tilde{\bv}_t\|}\big\rangle^2}{\big\langle \nabla f(\bx_{t}), \frac{\bx^\star-\bx_{t}}{\|\bx^\star-\bx_{t}\|}\big\rangle^2}\varepsilon_{t}
\end{align}We now use Theorem~\ref{thm:widthBound} to conclude the proof.
For Away-steps MP the proof is trivially extended since $2\min_{\bz\in\cA\cup-\cS}\langle \nabla f(\bx_t),\bz\rangle\leq \min_{\bz\in\cA, \bv\in\cS}\langle \nabla f(\bx_t),\bz-\bv\rangle$. Therefore, we obtain the same smoothness upper bound of the PWMP. The rest of the proof proceed as for PWMP with the additional $\frac12$ factor.

\paragraph{Sublinear Convergence for $\mu = 0$}
If $\mu = 0$ we have for PWMP:
\begin{align}\label{eq:PWsubBound}
f(\bx_{t+1})&\leq \min_{\gamma\leq \alpha_{\bv_t}} f(\bx_{t}) + \gamma\langle\nabla f(\bx_{t}), \tilde{\bz}_t - \tilde{\bv}_t\rangle \\&+ \frac{L}{2}\gamma^2\|\tilde{\bz}_t - \tilde{\bv}_t\|^2.
\end{align}
which can be rewritten for a good step (i.e. no clipping is necessary) as: 
\begin{eqnarray*}
 \varepsilon_{t+1} &\leq \varepsilon_{t} + \min_{\gamma\in[0,1]}\left\lbrace - \frac{\delta}{2} \gamma \varepsilon_{t} + \frac{\gamma^2}{2}L\rho^2\diam(\cA)^2\right\rbrace\\
\end{eqnarray*}
using the same arguments of the proof of Theorem~\ref{thm:NNMPsublinear}.
Unfortunately, $\alpha_{\bv_t}$ limits the improvement. On the other hand, we can repeat the induction of Theorem~\ref{thm:NNMPsublinear} for only the good steps.
Therefore:
\begin{eqnarray*}
 \varepsilon_{t+1}
 & \leq \varepsilon_{t} - \frac{2}{\delta' t+2}\delta' \varepsilon_{t} + \frac{1}{2}\left(\frac{2}{\delta' t+2}\right)^2 C,
\end{eqnarray*}
where we set $\delta' := \delta/2$, $C=L\rho^2\diam(\cA)^2$ and used $\gamma = \frac{2}{\delta' t+2} \in [0,1]$ (since it is a good step this produce a valid upper bound). Finally, we show by induction
 \begin{equation*}
 \varepsilon_t \leq \frac{4\left(\frac{2}{\delta} C + \varepsilon_0\right)}{t+4} = 2\frac{\left(\frac{1}{\delta'} C + \varepsilon_0\right)}{\delta' k(t)+2}
 \end{equation*}
where $k(t) \geq 0$ is the number of good steps at iteration $t$. 

When $k(t)=0$ we get $\varepsilon_0\leq \left(\frac{1}{\delta'}C+\varepsilon_0\right)$. Therefore, the base case holds. We now prove the induction step assuming $\varepsilon_t \leq \tfrac{2\left(\frac{1}{\delta'}C+\varepsilon_0\right)}{\delta' k(t)+2}$ as :
\begin{align*}
\varepsilon_{t+1} &\leq \left(1-\tfrac{2\delta'}{\delta' k(t) + 2}\right)\varepsilon_{t} + \tfrac12 C \left(\tfrac{2}{\delta' k(t)+2}\right)^2\\
&\leq \left(1-\tfrac{2\delta'}{\delta' k(t) + 2}\right)\tfrac{2\left(\frac{1}{\delta'}C+\varepsilon_0\right)}{\delta' k(t)+2} \\
&\quad+ \tfrac{1}{2}\left(\tfrac{2}{\delta' k(t)+2}\right)^2C + \tfrac{2}{(\delta' k(t)+2)^2}\delta'\varepsilon_0\\
&= \tfrac{2\left(\frac{1}{\delta'}C+\varepsilon_0\right)}{\delta' k(t)+2}\left(1-\tfrac{2\delta'}{\delta' k(t) +2}+\tfrac{\delta'}{\delta' k(t) +2}\right)\\
&\leq \tfrac{2\left(\frac{1}{\delta'}C+\varepsilon_0\right)}{\delta'(k(t)+1)+2}.
\end{align*}
For AFW the procedure is the same but the linear term of Equation~\ref{eq:PWsubBound} is divided by two. We proceed as before with the only difference that we call $\delta'=\delta/4$.
\end{proof}
\end{reptheorem}

\subsection{Proof sketch for linear rate convergence of FCMP}
\begin{reptheorem}{thm:PWMPlinear}
Let $\cA \subset \cH$ be a bounded set containing the origin
and let the objective function $f \colon \cH \to \R$ be both $L$-smooth and $\mu$-strongly convex over $\rho \conv(\cA\cup-\cA)$.

Then, the suboptimality of the iterates of Algorithm~\ref{algo:FCNNMP} decreases geometrically at each step in which $\gamma < \alpha_{\bv_t}$ (henceforth referred to as ``good steps'') as:
\begin{equation} 
\varepsilon_{t+1}
\leq \left(1- \beta \right)\varepsilon_{t},
\end{equation}
where $\beta := \delta^2\frac{\mu \cw^2}{L\diam(\cA)^2}\in (0,1]$, $\varepsilon_t := f(\bx_t) - f(\bx^\star)$ is the suboptimality at step $t$ and $\delta \in (0,1]$ is the relative accuracy parameter of the employed approximate \lmo  (Equation~\eqref{eq:inexactLMOMP}).
\begin{proof}
The proof is trivial noticing that:
\begin{align*}
f(\bx_{t+1}) &=\min_{\bx\in\cone(\cS\cup \bs(\cA,\br))}f(\bx)\\
&\leq \min_{\bx\in\cone(\cS\cup \bs(\cA,\br))}g_{\bx_t}(\bx)\\
&\leq \min_{\gamma\leq\alpha_{\bv_t}}g_{\bx_t}(\bx_t+\gamma(\bz_t-\bv_t))
\end{align*}
which is the beginning of the proof of Theorem~\ref{thm:PWMPlinear}. Note that there are no bad steps for variant 1. Since we minimize $f$ at each iteration, $\bv_t$ is always zero and each step is unconstrained (i.e., no bad steps).
\end{proof}
\end{reptheorem}

\section{Pyramidal Width}
Let us first recall some definitions from \citesup{LacosteJulien:2015wj}.
\paragraph{Directional Width}
\begin{align}
dirW(\cA,\br):=\max_{\bs,\bv\in\cA}\Big\langle\frac{\br}{\|\br\|},\bs-\bv\Big\rangle
\end{align}
\paragraph{Pyramidal Directional Width}
\begin{align}
PdirW(\cA,\br,\bx) := \min_{\cS\in\cS_\bx}dirW(\cS\cup
\lbrace\bs(\cA,\br)\rbrace,\br)
\end{align}
Where $\cS_\bx := \lbrace \cS  \ |\  \cS\subset \cA$ such that $\bx$ is a proper convex combination of all the elements in $\cS\rbrace$ and $\bs(\cA,\br) := \max_{\bs\in\cA}\langle\frac{\br}{\|\br\|},\bs\rangle$.
\paragraph{Pyramidal Width}
\begin{align*}
PWidth(\cA):= \min_{\substack{\cK\in \faces(\conv(\cA))\\ \bx\in \cK \\ \br\in\cone(\cK-\bx)\setminus \lbrace \0\rbrace}} PdirW(\cK\cap\cA,\br,\bx)
\end{align*}
Inspired by the notion of pyramidal width we now define the cone width of a set $\cA$.
\paragraph{Cone Width}
\begin{align*}
\cw:= \min_{\substack{\cK\in \gfaces(\cone(\cA))\\ \bx\in \cK \\ \br\in\cone(\cK-\bx)\setminus \lbrace \0\rbrace}} PdirW(\cK\cap\cA,\br,\bx)
\end{align*}

The linear rate analysis is dominated by the fact that, similarly as in FW, many step directions are constrained (the ones pointing outside of the cone). So these arguments are in line with \citesup{LacosteJulien:2015wj} and the techniques are adapted here. Lemma~\ref{lemma:faces} is a minor modification of [\citesup{LacosteJulien:2015wj}, Lemma 5], see also their Figure 3. If the gradient is not feasible, the vector with maximum inner product must lie on a facet. Furthermore, it has the same inner product with the gradient and with its orthogonal projection on that facet. While first proof of Lemma~\ref{lemma:faces} follows \citesup{LacosteJulien:2015wj}, we also give a different proof which does not use the KKT conditions. 

\begin{lemma}\label{lemma:faces}
Let $\bx$ be a reference point inside a polytope $\cK\in\gfaces(\cone(\cA))$ and $\br\in\lin(\cK)$ is not a feasible direction from $\bx$. Then, a feasible direction in $\cK$ minimizing the angle with $\br$ lies on a facet $\cK'$ of $\cK$ that includes $\bx$:
\begin{align*}
\max_{\be\in\cone(\cK-\bx)} \langle \br,\frac{\be}{\|\be\|}\rangle &= \max_{\be\in\cone(\cK'-\bx)} \langle \br,\frac{\be}{\|\be\|}\rangle \\&= \max_{\be\in\cone(\cK'-\bx)} \langle \br',\frac{\be}{\|\be\|}\rangle
\end{align*} 
where $\br'$ is the orthogonal projection of $\br$ onto $\lin(\cK')$
\begin{proof}
Let us center the problem in $\bx$.
We rewrite the optimization problem as:
\begin{align*}
\max_{\be\in\cone(\cK),\|\be\| = 1} \langle \br,\be\rangle
\end{align*}
and suppose by contradiction that $\be$ is in the relative interior of the cone. By the KKT necessary conditions we have that $\be^\star$ is collinear with $\br$. Therefore $\be^\star = \pm \br$. Now we know that $\br$ is not feasible, therefore the solution is $\be^\star = -\br$. By Cauchy-Schwarz we know that this solution is minimizing the inner product which is absurd. Therefore, $\be^\star$ must lie on a face of the cone. The last equality is trivial considering that $\br'$ is the orthogonal projection of $\br$ onto $\lin(\cK')$.
\paragraph{Alternative proof.}
This proof extends the traditional proof technique of \citesup{LacosteJulien:2015wj} to infinitely many constraints. We also reported the FW inspired proof for the readers that are more familiar with the FW analysis.
Using proposition 2.11 of \citesup{burger2003infinite} (we also use their notation) the first order optimality condition minimizing a function $J$ in a general Hilbert space given a closed set $\cK$ is that the directional derivative computed at the optimum $\bar u$ satisfy $J'(\bar u)v\geq 0$ $\forall v\in\cT(\cK-\bar u)$. Let us now assume that $\bar u $ is in the relative interior of $\cK$. Then $\cT(\cK-\bar u)=\cH$. Furthermore, $J'(\bar u)v = \langle\br,v\rangle$ which is clearly not greater or equal than zero for any element of $\cH$.
\end{proof}
\end{lemma}

Theorem~\ref{thm:widthBound} is the key argument to conclude the proof of Theorem~\ref{thm:PWMPlinear} from Equation~\eqref{step:linearRateCombinedfast}: we have to bound the ratio of those inner products with the cone width. 

\begin{theorem}\label{thm:widthBound}
Let $\br = -\nabla f(\bx_t)$, $\bx\in\cone(\cA)$, $\cS$ be the active set and $\bz$ and $\bv$ obtained as in Algorithm~\ref{algo:AMPPWMP}. Then, using the notation from Lemma \ref{lemma:faces}:
\begin{align}\label{eq:pyrThm}
\frac{\langle \br, \bd\rangle}{\langle \br,\hat\be\rangle}\geq \cw
\end{align}
where $\bd := \bz-\bv$, $\hat\be = \frac{\be}{\|\be\|}$ and $\be= \bx^\star -\bx_t$.
\begin{proof}
As we already discussed we can consider $\bx\in\conv(\cA)$ instead of $\bx_t\in\cone(\cA)$ since both the cone and the set of feasible direction are invariant to a rescaling of $\bx$ by a strictly positive constant.
Let us center all the vectors in $\bx$, then $\hat{\be}$ is just a vector with norm 1 in some face.
As $\bx$ is not optimal, by convexity we have that $\langle \br,\hat\be\rangle>0$. By Cauchy-Schwartz we know that $\langle \br,\hat\be\rangle\leq \|\br\|$ since $\langle \br,\hat\be\rangle>0$ and $\|\hat\be\| = 1$. 
By definition of $\bd$ we have:
\begin{align*}
\langle\frac{\br}{\|\br\|},\bd\rangle &= \max_{\bz\in\cA,\bv\in\cS} \langle\frac{\br}{\|\br\|},\bz-\bv\rangle\\
&\geq \min_{\cS\subset \cS_\bx}\max_{\bz\in\cA,\bv\in\cS} \langle\frac{\br}{\|\br\|},\bz-\bv\rangle\\
&=PdirW(\cA,\br,\bx).
\end{align*}

Now, if $\br$ is a feasible direction from $\bx$ Equation \eqref{eq:pyrThm} is proved (note that $PdirW(\cA,\br,\bx) \geq \cw$ as $\conv(\cA) \in \gfaces(\cone(\cA))$ and $\conv(\cA) \cap \cA = \cA$). If $\br$ is not a feasible direction it means that $\bx$ is on a face of $\cone(\cA)$ and $\br$ points to the exterior of $\cone(\cA)$ from $\bx$. 
We then project $\br$ on the faces of $\cone(\cA)$ containing $\bx$ until it is a feasible direction. We start by lower bounding the ratio of the two inner products replacing $\hat{\be}$ with a vector of norm 1 in the cone that has maximum inner product with $\br$ (with abuse of notation we still call it $\hat{\be}$).
We then write:
\begin{align*}
\frac{\langle \br, \bd\rangle}{\langle \bd,\hat\be\rangle}\geq \left(\max_{\bz\in\cA,\bv\in\cS} \langle{\br},\bz-\bv\rangle\right)\cdot\left(\max_{\be\in\cone(\cA-\bx)}\langle\br,\frac{\be}{\|\be\|}\rangle\right)^{-1}
\end{align*}
Let us assume that $\br$ is not feasible but without loss of generality is in $\lin(\cA)$ since orthogonal components to $\lin(\cA)$ does not influence the inner product with elements in $\lin(\cA)$. 
Using Lemma~\ref{lemma:faces} we know that:
\begin{align*}
\max_{\be\in\cone(\cK-\bx)} \langle \br,\frac{\be}{\|\be\|}\rangle &= \max_{\be\in\cone(\cK'-\bx)} \langle \br,\frac{\be}{\|\be\|}\rangle \\&= \max_{\be\in\cone(\cK'-\bx)} \langle \br',\frac{\be}{\|\be\|}\rangle
\end{align*} 
Let us now consider the reduced cone $\cone(\cK')$ as $\br\in\lin(\cK')$. For the numerator we obtain:
\begin{align*}
\max_{\bz\in\cA,\bv\in\cS} \langle{\br},\bz-\bv\rangle
&\stackrel{\cK'\subset\cA}{\geq} \max_{\bz\in\cK'} \langle{\br},\bz\rangle + \max_{\bv\in\cS} \langle{-\br},\bv\rangle
\end{align*}
Putting numerator and denominator together we obtain:
\begin{align*}
\frac{\langle \br, \bd\rangle}{\langle \bd,\hat\be\rangle}\geq \left(\max_{\substack{\bz\in\cK'\\\bv\in\cS}}\langle{\br'},\bz-\bv\rangle\right)\cdot\left(\max_{\be\in\cone(\cK'-\bx)} \langle \br',\frac{\be}{\|\be\|}\rangle\right)^{-1}
\end{align*}
Note that $\cS\subset\cK'$. Indeed, $\bx$ is a proper convex combination of the elements of $\cS$ and $\bx\in\cK'\subset\conv(\cA)$.
Now if $\br'$ is a feasible direction in $\cone(\cK'-\bx)$ we obtain the cone width since $\cone(\cK')$ is a face of $\cone(\cA)$. If not we reiterate the procedure projecting onto a lower dimensional face $\cK^{''}$. Eventually, we will obtain a feasible direction. Since $\langle \br,\hat\be\rangle\neq 0$ we will obtain $\br_{final}\neq \0$.  
\end{proof}
\end{theorem}

\paragraph{Lemma~\ref{lem:mdw}.}
\textit{If the origin is in the relative interior of $\conv(\cA)$ with respect to its linear span, then $\cone(\cA)=\lin(\cA)$ and $\cw= \mdw$.}
\begin{proof}
Let us first rewrite the definition of cone width:
\begin{align*}
\cw:= \min_{\substack{\cK\in \gfaces(\cone(\cA))\\ \bx\in \cK \\ \br\in\cone(\cK-\bx)\setminus \lbrace \0\rbrace}} PdirW(\cK\cap\cA,\br,\bx).
\end{align*}
The minimum is over all the feasible directions of the gradient from every point in the domain.
It is not restrictive to consider $\br$ parallel to $\lin(\cA)$ (because the orthogonal component has no influence). Therefore, from every point $\bx\in\lin(\cA)$ every $\br\in\lin(\cA)$ is a feasible direction.
The geometric constant then becomes:
\begin{align*}
\cw=\min_{\substack{\cK\in \gfaces(\cone(\cA)) \\ \bx\in\cK\\\br\in\lin(\cA)\setminus \lbrace \0\rbrace}} PdirW(\cK\cap\cA,\br,\bx)
\end{align*}

 Let us now assume by contradiction that for any $\cK\in\gfaces$ we have: 
\begin{align}\label{eq:minxface}
\0\not\in \argmin_{\bx\in\cK}\min_{\br\in\lin(\cA)\setminus\lbrace \0\rbrace}PdirW(\cK\cap\cA,\br,\bx)
\end{align} 
Therefore, $\exists\bv\in\cS$ such that $\bv\neq \0$ for any of the $\bx$ minimizing~\eqref{eq:minxface}. By definition, we have $\0\in\cS$, which yields $\max_{\bv\in\cS}\langle \br,-\bv\rangle\geq0$ for every $\br$. Therefore, $\langle \br,\bz-\bv\rangle\geq\langle \br,\bz\rangle$ which is absurd because we assumed zero was in the set of minimizers of \eqref{eq:minxface}.
 So $\0$ minimize the cone directional width which yields $\cS_\bx = \lbrace \0 \rbrace$ and $\bv=\0$.
 In conclusion we have:
\begin{align*}
\cw = \min_{\bd\in\lin(\cA)} \max_{\bz\in\cA} \langle \frac{\bd}{\|\bd\|},\bz\rangle = \mdw
\end{align*}
\end{proof}

\section{Affine Invariant Sublinear Rate}
\begin{reptheorem}{thm:sublinearNNMPAffineInvariant}
Let $\cA \subset \cH$ be a bounded set with $\0\in\cA$, $\rho := \max\left\lbrace \|\bx^\star\|_{\cA}, \|\bx_{0}\|_{\cA},\ldots,\|\bx_T\|_{\cA}\right\rbrace < \infty$. Assume $f$ has smoothness constant $\CfMPr$.
Then, Algorithm~\ref{algo:NNMPaffine} converges for $t \geq 0$ as 
\[
f(\bx_t) - f(\bx^\star)\leq \frac{4\left(\frac2\delta \CfMPr+\varepsilon_0\right)}{\delta t+4} ,
\]
where $\delta \in (0,1]$ is the relative accuracy parameter of the employed approximate \lmo~\eqref{eq:inexactLMOMP}.
\end{reptheorem}

\begin{proof}
Recall that $\tilde{\bz}_t$ is the atom returned by the inexact \lmo after the comparison with $-\frac{\bx_t}{\|\bx_t\|_\cA}$ at the current iteration $t$. 

We start by upper bounding $f$ on $\rho \conv(\bar\cA)$ using smoothness as follows:
\begin{equation*}
f(\bx_{t+1})
\leq   \min_{\gamma\in[0,1]}f(\bx_t)+\gamma\langle\nabla f(\bx_t),\rho\tilde\bz_t \rangle + \frac{\gamma^2}{2}\CfMPr
\end{equation*}
We now proceed bounding the linear term as done in the proof of Theorem~\ref{thm:NNMPsublinear} for \textbf{case A} and \textbf{case B} obtaining: 
\begin{align*}
 f(\bx_{t+1})\leq f(\bx_t) +  \min_{\gamma\in[0,1]} \left\lbrace - \frac{\delta}{2} \gamma \varepsilon_t + \frac{\gamma^2}{2}\CfMPr \right\rbrace
\end{align*}

Now, subtracting $f(\bx^\star)$ from both sides we get
\begin{eqnarray*}
 \varepsilon_{t+1} &\leq \varepsilon_{t} + \min_{\gamma\in[0,1]}\left\lbrace - \frac{\delta}{2} \gamma \varepsilon_{t} + \frac{\gamma^2}{2}\CfMPr\right\rbrace\\
 & \leq \varepsilon_{t} - \frac{2}{\delta' t+2}\delta' \varepsilon_{t} + \frac{1}{2}\left(\frac{2}{\delta' t+2}\right)^2 \CfMPr,
\end{eqnarray*}
where we set $\delta' := \delta/2$ and used $\gamma = \frac{2}{\delta' t+2} \in [0,1]$ to obtain the second inequality. Finally, we show by induction
 \begin{equation*}
 \varepsilon_t \leq \frac{4\left(\frac{2}{\delta} \CfMPr + \varepsilon_0\right)}{t+4} = 2\frac{\left(\frac{1}{\delta'} \CfMPr + \varepsilon_0\right)}{\delta' t+2}
 \end{equation*}
for $t \geq 0$.

When $t=0$ we get $\varepsilon_0\leq \left(\frac{1}{\delta'}\CfMPr+\varepsilon_0\right)$. Therefore, the base case holds. We now prove the induction step assuming $\varepsilon_t \leq \tfrac{2\left(\frac{1}{\delta'}\CfMPr+\varepsilon_0\right)}{\delta' t+2}$ as :
\begin{align*}
\varepsilon_{t+1} &\leq \left(1-\tfrac{2\delta'}{\delta' t + 2}\right)\varepsilon_{t} + \tfrac12 \CfMPr \left(\tfrac{2}{\delta' t+2}\right)^2\\
&\leq \left(1-\tfrac{2\delta'}{\delta' t + 2}\right)\tfrac{2\left(\frac{1}{\delta'}\CfMPr+\varepsilon_0\right)}{\delta' t+2} \\
&\quad+ \tfrac{1}{2}\left(\tfrac{2}{\delta' t+2}\right)^2\CfMPr + \tfrac{2}{(\delta' t+2)^2}\delta'\varepsilon_0\\
&= \tfrac{2\left(\frac{1}{\delta'}\CfMPr+\varepsilon_0\right)}{\delta' t+2}\left(1-\tfrac{2\delta'}{\delta' t +2}+\tfrac{\delta'}{\delta' t +2}\right)\\
&\leq \tfrac{2\left(\frac{1}{\delta'}\CfMPr+\varepsilon_0\right)}{\delta'(t+1)+2}.
\end{align*}
\end{proof}
We next explore the relationship of $\CfMPr$ and the smoothness parameter. Recall that $f$ is \emph{$L$-smooth} with respect to a given norm $\|.\|$ over a set $\cQ$ if 
\begin{equation}
\| \nabla f(\bx) - \nabla f(\by) \|_* \leq L \| \bx-\by \| \text{ for all } \bx,\by\in\cQ \ ,
\end{equation}
 where $\|.\|_*$ is the dual norm of $\|.\|$.

\begin{lemma}
\label{lem:CfwithRadius}
Assume $f$ is $L$-smooth with respect to a given norm $\|.\|$, over the set $\conv(\cA)$.
Then, 
\begin{equation}
 \CfMPr \leq L \, \rho^2\radius_{\norm{.}}(\cA)^2
\end{equation}
\end{lemma}
\begin{proof}
By the definition of smoothness of $f$ with respect to $\|.\|$, 
\begin{eqnarray*}
D(\by,\bx) \leq \frac{L}{2} \| \by - \bx\|^2.
\end{eqnarray*}

Hence, from the definition of $\CfMPr$, 
\begin{eqnarray*}
\CfMP &\leq&\sup_{\substack{\bs\in\rho\cA,\bx\in\conv(\rho\cA)\\\gamma\in [0,1]\\\by = \bx+\gamma\bs}}
 \frac{2}{\gamma^2} \frac{L}{2}  \| \by - \bx\|^2 \\ 
&=& L \rho^2\sup_{\bs \in {\cA}} \, \| \bs\|^2 \\
&=& L \, \rho^2\radius_{\norm{.}}(\cA)^2 \ . 
\end{eqnarray*}
\end{proof}
\section{Affine Invariant Linear Rate}\label{sec:affineLinearProof}

\begin{reptheorem}{thm:LinearMPAffineInvariant}
Let $\cA \subset \cH$ be a bounded set containing the origin
and let the objective function $f \colon \cH \to \R$ have smoothness constant $\CfMPrPW$ and strong convexity constant $\mufr$

Then, the suboptimality of the iterates of Algorithm~\ref{algo:AMPPWMP} and \ref{algo:FCNNMP} decreases geometrically at each step in which $\gamma < \alpha_{\bv_t}$ (henceforth referred to as ``good steps'') as:
\begin{equation} 
\varepsilon_{t+1}
\leq \left(1- \beta \right)\varepsilon_{t},
\end{equation}
where $\beta := \delta^2\frac{\mufr}{\CfMPrPW}\in (0,1]$, $\varepsilon_t := f(\bx_t) - f(\bx^\star)$ is the suboptimality at step $t$ and $\delta \in (0,1]$ is the relative accuracy parameter of the employed approximate \lmo  (Equation~\eqref{eq:inexactLMOMP}). For AMP (Algorithm~\ref{algo:AMPPWMP}), $\beta^{\text{AMP}} = \beta/2$.
If $\mufr = 0$ Algorithm~\ref{algo:AMPPWMP} converges with rate $O(1/k(t))$ where $k(t)$ is the number of ``good steps'' up to iteration t.
\vspace{-1mm}
\begin{proof}
Let us first consider the PWMP update.
Using the definition of $\CfMPrPW$ we upper-bound $f$ on $\rho \conv(\cA)$ as follows
\begin{eqnarray*}
f(\bx_{t+1}) &\leq & \min_{\gamma \in [0,1]} f(\bx_t) + \gamma \langle \nabla f(\bx_t), \rho \tilde{\bz}_t -\rho \tilde{\bv}_t \rangle   \\&+& \frac{\gamma^2}{2} \CfMPrPW \\
&= & \min_{\gamma \in \bbR} f(\bx_t) + \gamma \langle \nabla f(\bx_t), \rho \tilde{\bz}_t - \rho \tilde{\bv}_t \rangle  \\&+& \frac{\gamma^2}{2} \CfMPrPW \\
& =& f(\bx_t) - \frac{\rho^2}{2\CfMPrPW} \left\langle \nabla f(\bx_t), \tilde{\bz}_t- \tilde{\bv}_t \right\rangle ^2.
\end{eqnarray*}
This upper bound holds for Algorithm~\ref{algo:AMPPWaffine} every time $\rho\gamma<\alpha_{\bv}$  as $\rho\gamma$ minimizing the RHS of the first equality coincides with the update of Algorithm~\ref{algo:NNMPaffine} Line 5. The first equality holds as $\CfMPrPW$ is defined on $\rho \conv(\cA)$ and $\rho \conv(\cA)$ contains all iterates by definition, so that the unconstrained minimum lies in $[0,1]$ assuming $\rho\gamma<\alpha_{\bv}$. 

Using $\varepsilon_t = f(\bx^\star) - f(\bx_t)$, we can lower bound the error decay as 
\begin{eqnarray}\label{eqproof:linearMPAffineLowerBound}
\varepsilon_{t} - \varepsilon_{t+1} \geq \frac{\rho^2}{2\CfMPrPW} \left\langle \nabla f(\bx_t), \tilde{\bz}_t - \tilde{\bv}_t\right\rangle ^2.
\end{eqnarray}

Starting from the definition of $\mufr$ we get,
\begin{eqnarray*}
\frac{\gamma(\bx_t, \bx^\star)^2}{2} \mufr &\leq & f(\bx^\star) - f(\bx_t) - \langle \nabla f(\bx_t), \bx^\star - \bx_t)\rangle \\ 
& = & -\varepsilon_t \\&+& \gamma(\bx_t, \bx^\star) \langle-\nabla f(\bx_t), \bs(\bx_t)-\bv(\bx)\rangle ,
\end{eqnarray*}
which gives 
\begin{eqnarray}\nonumber
\varepsilon_t &\leq & - \frac{\gamma(\bx_t, \bx^\star)^2}{2} \mufr  \\&+& \gamma(\bx_t, \bx^\star)  \langle-\nabla f(\bx_t), \bs(\bx_t)-\bv(\bx)\rangle  \\\nonumber
& \leq &  \frac{\langle-\nabla f(\bx_t), \bs(\bx_t)-\bv(\bx)\rangle^2} {2 \mufr} \\ \label{eqproof:linearMPAffineUpperBound}
& = & \frac{\langle-\nabla f(\bx_t), \rho(\tilde{\bz}_t-\tilde{\bv}_t)\rangle^2} {2 \delta^2 \mufr} 
\end{eqnarray}
where the last inequality is by the quality of the approximate LMO as used in the algorithm, as defined in \eqref{eq:inexactLMOMP}.

Combining equations~\eqref{eqproof:linearMPAffineLowerBound} and \eqref{eqproof:linearMPAffineUpperBound}, we have
\begin{eqnarray*}
\varepsilon_{t} - \varepsilon_{t+1} \geq \delta^2 \frac{\mufr}{\CfMPrPW} \, \varepsilon_t ,
\end{eqnarray*}
which proves the claimed result. 
The proof for AMP and FCMP follows directly using the same argument used in the proof of Theorem~\ref{thm:PWMPlinear}. The upper bound used in the FCMP is the affine invariant notion of smoothness.
The proof steps for the sublinear convergence is the same as the one of Theorem~\ref{thm:PWMPlinear} replacing $C$ with $\CfMPrPW$.
\end{proof}

\end{reptheorem}
\begin{lemma}
If $f$ is $\mu$ strongly convex over the domain $\conv(\rho\cA)$ with respect to some arbitrary cholsen norm $\|\cdot\|$, then
\begin{align*}
\mufr\geq\mu\cw^2
\end{align*}
\end{lemma}
\begin{proof}
From the strong convexity:
\begin{align*}
\mufr &= \inf_{\bx \in \conv(\rho\cA)} \inf_{\substack{\bx^\star \in \conv(\rho\cA)\\ \langle\nabla f (\bx), \bx^\star - \bx \rangle < 0 }} \frac{2}{\gamma(\bx,\bx^\star)} D(\bx^\star,\bx)\\
&\geq \inf_{\substack{\bx,\bx^\star \in\conv(\rho\cA),\\ \langle-\nabla f(\bx),  \bx^\star - \bx \rangle>0 }}\mu\left(\frac{\langle-\nabla f(\bx),  \bs(\bx)-\bv(\bx) \rangle}{\langle-\nabla f(\bx),  \frac{\bx^\star - \bx}{\|\bx^\star-\bx\|_\cA} \rangle}\right)^2\\
&\geq \mu\cw^2
\end{align*}
where in the last inequality we used Theorem~\ref{thm:widthBound}.

The proof for away-steps uses the same argument we used in the norm based rate.
\end{proof}
\begin{lemma}
Assume $f$ is $L$-smooth with respect to a given norm $\|.\|$, over the set $\conv(\cA)$.
Then, 
\begin{equation}
 \CfMPrPW \leq L \, \rho^2\diam_{\norm{.}}(\cA)^2
\end{equation}
\end{lemma}
\begin{proof}
By the definition of smoothness of $f$ with respect to $\|.\|$, 
\begin{eqnarray*}
D(\by,\bx) \leq \frac{L}{2} \| \by - \bx\|^2.
\end{eqnarray*}

Hence, from the definition of $\CfMPr$, 
\begin{eqnarray*}
\CfMP &\leq&\sup_{\substack{\bs\in\rho\cA,\bx\in\conv(\rho\cA)\\\bv\in\cS\\\gamma\in [0,1]\\\by = \bx+\gamma(\bs-\bv)}}
 \frac{2}{\gamma^2} \frac{L}{2}  \| \by - \bx\|^2 \\ 
&=& L \rho^2\sup_{\substack{\bx\in\conv(\rho\cA)\\\bs \in {\cA}\\\bv\in\cS}} \, \| \bs-\bv\|^2 \\
&=& L \, \rho^2\diam_{\norm{.}}(\cA)^2 \ . 
\end{eqnarray*}
\end{proof}

\setlength{\bibsep}{3.5pt plus .5ex}
{\small
\bibliographystylesup{plain}
\bibliographysup{bibliography}
}

\end{document}